%% file: stable-compression.tex
\documentclass[12pt,breaklinks]{alt2021blank}

\title[Stable Sample Compression Schemes]{Stable Sample Compression Schemes: New Applications and an\\ Optimal SVM Margin Bound}

\usepackage{times}

\altauthor{%
 \Name{Steve Hanneke} \Email{steve.hanneke@gmail.com}\\
 \addr Toyota Technological Institute at Chicago
 \AND
 \Name{Aryeh Kontorovich} \Email{karyeh@cs.bgu.ac.il}\\
 \addr Ben-Gurion University%
}
\usepackage{url}

\usepackage{color}

\input{macros}

\renewenvironment{proof}[1][]{\par\noindent{\bf Proof #1\ }}{\hfill\BlackBox\\[2mm]}

\begin{document}

\maketitle

\begin{abstract}
\input{abstract}
\end{abstract}

\section{Introduction}
\label{sec:intro}

The recent work of \citet*{bousquet:20,bousquet:20b} 
introduced a new technique for proving 
PAC generalization guarantees for a special 
type of compression scheme, called a \emph{stable} 
compression scheme: namely, a compression scheme 
$(\kappa,\rho)$ for which any $S, S' \in (\X \times \Y)^*$
with $\kappa(S) \subseteq S' \subseteq S$ has 
$\rho(\kappa(S')) = \rho(\kappa(S))$.
They proved that for any stable compression scheme of 
size $k$, it holds with probability at least $1-\delta$ 
over the draw of $n$ i.i.d.\ samples $S$,  
that the risk 
$R(\rho(\kappa(S))) = O\!\left( \frac{k}{n} + \frac{1}{n}\log\!\left(\frac{1}{\delta}\right)\right)$, 
provided that the sample $S$ is guaranteed to be 
separable by the compression scheme: 
that is, the empircal risk $\hat{R}_{S}(\rho(\kappa(S))) = 0$.
This 
presents
an improvement over the traditional 
analysis of general compression schemes, which has 
an additional $\log n$ factor on the first term
\citep*{littlestone:86,floyd:95}.
They used this result to provide new PAC generalization 
guarantees for a variety of learning algorithms and 
techniques, including establishing that the 
\emph{Support Vector Machine} (SVM) 
algorithm for learning linear separators in 
$\reals^d$ obtains the minimax optimal bound 
$O\!\left( \frac{d}{n} + \frac{1}{n}\log\!\left(\frac{1}{\delta}\right)\right)$, 
which resolved a long-standing open question 
dating back to the early work of \citet*{vapnik:74}.

In the present work, we explore further implications 
of this technique for analyzing stable compression schemes.
Our contributions are of two types.
First, while the applications discussed by 
\citet*{bousquet:20,bousquet:20b} focused on 
data-independent generalization bounds, 
in the present work 
we investigate further implications of this technique 
toward providing improved \emph{data-dependent} 
PAC generalization 
bounds for several learning algorithms.  
In particular, one of the results we prove in this 
vein is a sharp \emph{margin} bound, which holds 
for both SVM and Perceptron: namely, a bound 
(holding with probability at least $1-\delta$) 
of the form 
\begin{equation*} 
R_{\PXY}(\hat{h}) = O\!\left( \frac{r^2}{\gamma^2}\frac{1}{n} + \frac{1}{n}\log\!\left(\frac{1}{\delta}\right)\right),
\end{equation*}
where $\gamma$ is the geometric margin and $r$ is the 
radius of the data.
This refines all previous existing margin bounds for SVM 
by a log factor, and is provably \emph{optimal}.
Establishing a bound of this form has been a known 
open question in the literature for many years 
\citep*[see][for background]{hanneke:19b}.
Our results on data-dependent PAC generalization 
bounds also include a new tighter data-dependent 
bound for all empirical risk minimization algorithms  
in general VC classes.  We also prove a new tighter 
bound on the probability in the region of disagreement 
of version spaces, which has implications for 
the analysis of disagreement-based active learning 
algorithms.

A second type of extension we provide is to establish 
new PAC generalization bounds for stable compression 
schemes in the \emph{agnostic} setting: that is, 
where $\hat{R}_{S}(\rho(\kappa(S)))$ may be non-zero.
We specifically prove that, with probability at least 
$1-\delta$, 
$| \hat{R}_{S}(\rho(\kappa(S))) - R(\rho(\kappa(S))) | 
= O\!\left( \sqrt{ \frac{1}{n} \left( |\kappa(S)| + \log\!\left(\frac{1}{\delta}\right) \right)}\right)$, 
which is provably better than is achievable 
by general (non-stable) sample compression schemes in the 
agnostic setting.
We additionally establish a new Bernstein-type bound 
for stable sample compression schemes, 
of the form 
$| \hat{R}_{S}(\rho(\kappa(S))) - R(\rho(\kappa(S))) | 
=$ $O\!\left( \sqrt{ \hat{R}_{S}(\rho(\kappa(S))) \frac{1}{n} \left( |\kappa(S)| + \log\!\left(\frac{1}{\delta}\right) \right)} + \frac{1}{n} \left( |\kappa(S)| + \log\!\left(\frac{1}{\delta}\right) \right)\right)$.
As a concrete implication of these general results, 
we provide a sharper generalization bound for the 
technique of \emph{compressed nearest neighbor} prediction 
studied by \citet*{hanneke:19a}.

\section{Main Results}
\label{sec:main-results}

This section provides formal statements of the 
main results of this work.

\subsection{Notation}
\label{sec:notation}

Before stating our results more formally, we introduce 
some basic notation to be used throughout.
We denote by $\X$ an \emph{instance space}: 
a non-empty set equipped with a $\sigma$-algebra
specifying the measurable subsets.
Denote by $\Y$ a \emph{label space}.
For our general results, $\Y$ may be any non-empty set
(equipped with a $\sigma$-algebra specifying the measurable 
subsets), though our applications will focus on the case 
$\Y = \{-1,1\}$, corresponding to \emph{binary classification}.
We refer to any measurable function $h : \X \to \Y$ as 
a \emph{classifier}.
For any distribution $\PXY$ on $\X \times \Y$, 
define the \emph{risk} of a classifier $h$ by
$R_{\PXY}(h) = \PXY( (x,y) : h(x) \neq y )$.
For any finite \emph{data set} $S \in (\X \times \Y)^*$, 
also define the \emph{empirical risk}:  
$\hat{R}_{S}(h) = \frac{1}{|S|} \sum_{(x,y) \in S} \ind[ h(x) \neq y ]$.  
Also, in the results below, we use the convention
$\log(x) = \max\{\ln(x),1\}$ for any $x \geq 0$.

\subsection{Optimal Margin Bounds for SVM and Perceptron}
\label{sec:margin-summary}

Linear classifiers lie at the very foundations of modern machine learning theory \citep{Aizerman67theoretical} --- and as such, their risk rates have been an active topic of research.
The agnostic case is well-understood:
if all but a few of $n$ labeled data points residing on the $\dim$-dimensional unit sphere
are linearly separated with margin at least $\gamma$
(the few exceptions being treated as sample errors),
then the expected excess risk decays as
\citep{MR1383093,mohri-book2012}
$
\Theta\paren{\sqrt{{\min\set{\dim,1/\gamma^2}}/{n}}}
.
$
For the separable case, in which 
there exists a hyperplane in $\R^\dim$
consistent with the $n$ sample points and having margin at least $\gamma$,
the best guarantee on the expected risk by any learning algorithm
is lower-bounded by
$\Omega(\min\{\dim,1/\gamma^2\}/n)$.
Similarly,
any generalization bound that holds
with probability $1-\delta$ 
is lower bounded by
$
\Omega\left(  \left( \min\{\dim,1/\gamma^{2}\} + \log\left({1}/{\delta}\right) \right)/n \right)
$.
Nearly matching upper bounds are readily available via standard VC theory and Rademacher analysis, but these have additional $\log n$ factors.
For weaker in-expectation bounds,
the upper and lower bounds match up to constants.
Since
\citet{vapnik:74}, where the margin-based
in-expectation bound was first stated,
it has been an open problem to
extend these to tight high-probability bounds.
(We refer the reader to \citet{DBLP:journals/tcs/HannekeK19} for
proofs of the aforementioned claims 
as well as comprehensive background.) Recently, \citet{bousquet:20}
obtained the sharp high-probability upper bound
in terms of the dimension $\dim$.

\paragraph{Our contribution.}
We prove an optimal PAC margin bound for SVM.
This matches minimax lower bounds, and is 
therefore the first proof that SVM achieves the optimal 
margin bound (previously only known to be achievable 
by certain online-to-batch conversion techniques 
applied to Perceptron).

For any $d \in \nats$ and 
data set $S \in (\reals^{d} \times \{-1,1\})^{*}$, 
let $r(S) := \max_{(\x,y) \in S} \|\x\|$ 
and let 
\[
\gamma(S) := \max_{\w \in \reals^{d}, b \in \reals : \|\w\|=1} \min_{(\x,y) \in S} y ( \w \cdot \x + b).
\]
Define $\hat{h}_{{\rm SVM}} = {\rm SVM}(S)$ 
as the function 
$\hat{h}_{{\rm SVM}}(x) = \sign(\hat{\w} \cdot \x + \hat{b})$ 
where $\hat{\w},\hat{b}$ realize the max in the above 
definition.

\begin{theorem}
\label{thm:svm-margin}
For any distribution $\PXY$, any $n \in \nats$, 
and any $\delta \in (0,1)$, 
for $S \sim \PXY^{n}$, 
with probability at least $1-\delta$, 
if $S$ is linearly separable, 
then for $\hat{h}_{{\rm SVM}} = {\rm SVM}(S)$, 
letting $r = r(S)$ and $\gamma = \gamma(S)$, 
\begin{equation*}
R_{\PXY}(\hat{h}_{{\rm SVM}}) 
= O\!\left( \frac{r^2}{\gamma^2}\frac{1}{n} + \frac{1}{n} \log\!\left( \frac{1}{\delta} \right) \right).
\end{equation*}
\end{theorem}

This matches (up to numerical constants) a known 
minimax lower bound, and therefore establishes 
that SVM achieves the optimal PAC margin bound 
(up to numerical constants).  This optimal 
margin bound was previously only known for a 
certain (more-involved) 
online-to-batch conversion technique 
of \citet*{littlestone:89} applied to the Perceptron 
algorithm.

Additionally, we extend a result on online-to-batch 
conversion to prove that the \emph{Perceptron} online 
learning algorithm also achieves the \emph{same} 
optimal PAC margin bound.  Again, this bound was 
previously only known to be achievable for certain 
modified variants of Perceptron (e.g., which record all 
the intermediate classifiers and select on based on a 
hold-out set).  Here we show that, if we simply take 
$\hat{h}_{p}$ as the \emph{final} predictor from cycling 
Perceptron through the data set until it makes a 
complete pass without making any mistakes, then 
$\hat{h}_{p}$ also achieves the optimal margin bound.
Formally, the algorithm is defined as follows.

\begin{definition}[Perceptron]
\label{def:perceptron}
If $S$ is a linearly separable data set,
then the ${\rm Perceptron}$ algorithm \citep{Rosenblatt1958}, 
denoted by $\perc$, initializes
$(\w,b)\in\R^{\dim+1}$ to $\boldsymbol{0}$
and cycles through $S$ (in order),
evaluating $\sign(\w\cdot\x_i+b)$
on each data point. On each mistake
(i.e., $\sign(\w\cdot\x_i+b)\neq y_i$),
$(\w,b)$ is updated via the rule $\w\leftarrow
\w+y_i\x_i$ and $b \leftarrow b+y_i$.
\end{definition}

The key property of Perceptron that enables us 
to obtain margin bounds is the following.
For data with $x$ components lying in a ball 
of radius $r$,
and linearly separable with margin $\gamma$,
Perceptron was shown by \citet*{MR0175722} 
to make at most $\frac{r^2+1}{\gamma^2}$ mistakes 
($\perc$ does not terminate on non-separable inputs).
Here the ``$+1$'' in the numerator 
is accounting for applying the 
rule with a bias term $b$; for homogeneous separators, 
the number of mistakes is at most $\frac{r^2}{\gamma^2}$.
The claim for the non-homogeneous case is obtained 
by reduction to the homogeneous case, increasing the 
dimension by $1$ and letting each $x$ have coordinate 
$\dim+1$ fixed to $1$, in which case the radius 
of the data in $\dim+1$ dimensions is at most 
$\sqrt{r^2+1}$, and 
the margin of the data with respect to 
non-homogeneous separators is precisely 
the margin of the $\dim+1$ dimensional augmented 
data with respect to homogeneous separators.

We prove the following result for this algorithm 
(where $r(S)$ and $\gamma(S)$ are as defined above).
The proof is included in Section~\ref{sec:perceptron-proof}.

\begin{theorem}
\label{thm:perceptron-margin}
For any distribution $\PXY$, any $n \in \nats$, 
and any $\delta \in (0,1)$, 
for $S \sim \PXY^{n}$, 
with probability at least $1-\delta$,
if $S$ is linearly separable,
then letting $r = r(S)$ and $\gamma = \gamma(S)$, 
the classifier $\hat{h}_p = \perc(S)$ satisfies 
\begin{equation*}
R_{\PXY}(\hat{h}_{p}) 
= O\!\left( \frac{r^2}{\gamma^2}\frac{1}{n} + \frac{1}{n} \log\!\left( \frac{1}{\delta} \right) \right).
\end{equation*}
\end{theorem}

We note that this result is stronger than 
the analogous PAC bounds known for using Perceptron 
with other well-known online-to-batch conversion techniques, 
such as the ``longest survivor'' technique
\citep*{kearns:87,gallant:90} 
or voted Perceptron \citep*{freund:99}.
It also matches (up to constants) the result 
of \citet*{littlestone:89}, which was for a 
considerably more-involved online-to-batch conversion 
technique, which keeps all of the intermediate 
hypotheses and in the end selects one using 
a held-out portion of the data.

\subsection{The Probability in the Region of Disagreement of a Version Space}
\label{sec:PDIS-results}

Fix any measurable class of functions: $\H \subseteq \Y^{\X}$.
For any $n \in \nats$ and $S \in (\X \times \Y)^n$, 
define $\H[S] = \{ h \in \H : \hat{R}_{S}(h) = 0 \}$.
Fix any distribution $\PXY$ on $\X \times \Y$, 
let $(X_1,Y_1),(X_2,Y_2),\ldots$ be i.i.d.\ 
with distribution $\PXY$, 
and for any $n \in \nats$ let $S_n = \{(X_i,Y_i)\}_{i=1}^{n}$.
Define the \emph{version space} 
$V_n = \H[S_n]$
\citep*{mitchell:77}, 
and define its 
\emph{region of disagreement} 
$\DIS(V_n) = \{ x \in \X : \exists h,h' \in V_n, h(x) \neq h'(x) \}$
\citep*{cohn:94,balcan:06,hanneke:fntml}.
The set $\DIS(V_n)$ plays an important role 
in certain disagreement-based active learning algorithms, 
and the analysis thereof: most-notably, the 
\emph{CAL} active learning algorithm introduced by 
\citet*{cohn:94} and studied theoretically 
in great detail by a large number of works 
\citep*[e.g.,][]{hanneke:07b,hanneke:thesis,hanneke:11a,hanneke:12a,hanneke:fntml,hanneke:16b,hanneke:07c,hanneke:10a,hsu:thesis,el-yaniv:10,el-yaniv:12,hanneke:15a,hanneke:15b}.
Of particular interest in quantifying the 
\emph{label complexity} of CAL is bounding 
$\Px(\DIS(V_n))$ as a function of $n$.
Classic well-known bounds on this quantity were 
established by \citet*{hanneke:07b,hanneke:thesis,hanneke:11a} 
based on a quantity known as the \emph{disagreement coefficient}
\citep*[see also][]{hanneke:fntml}.
However, more-recently \citet*{hanneke:15a} and 
\citet*{hanneke:16b} established bounds that are 
sometimes tighter, based on a quantity $\hat{t}_n$ called the 
\emph{version space compression set size} 
\citep*{el-yaniv:10}
(also known as the \emph{empirical teaching dimension} 
in earlier work of \citealp*{hanneke:07a}).
Specifically, define 
\begin{equation*}
\hat{t}_n = \min\!\left\{ |S'| : S' \subseteq S_n, 
\H[S'] = V_n \right\}. 
\end{equation*}
That is, $\hat{t}_n$ is the size of the smallest 
subset of $S_n$ that induces the same version space.

Based on stable compression arguments, 
\citet*{bousquet:20b} prove a data-independent 
bound on $\Px(\DIS(V_n))$ in terms of a 
combinatorial quantity called the \emph{star number} 
from \citet*{hanneke:15b}, which improved the 
numerical constant factors compared to a result of 
the same form established by \citet*{hanneke:16b}.
However, they did not discuss data-dependent 
or distribution-dependent bounds on $\Px(\DIS(V_n))$.
In particular, since $\hat{t}_n$ is never larger 
than the star number, and is often strictly smaller, 
bounds on $\Px(\DIS(V_n))$ 
based on $\hat{t}_n$ may be considered stronger 
than bounds based on the star number.
While prior works by \citet*{hanneke:15a} and 
\citet*{hanneke:16b} have established bounds on 
$\Px(\DIS(V_n))$ in terms of $\hat{t}_n$, 
both of these results have a suboptimal form 
(as we discuss in detail in Section~\ref{sec:PDIS}).
In the present work, 
we prove the following result, which improves over 
all of these previously known bounds on $\Px(\DIS(V_n))$ 
based on $\hat{t}_n$.  The proof is presented 
in Section~\ref{sec:PDIS}.

\begin{theorem}
\label{thm:PDIS-coarse}
For any $n \in \nats$ and $\delta \in (0,1)$, 
with probability at least $1-\delta$, 
\begin{equation*}
\Px(\DIS(V_n)) = O\!\left( \frac{1}{n} \left( \hat{t}_n + \log\!\left(\frac{1}{\delta}\right) \right) \right).
\end{equation*}
\end{theorem}

\subsection{A New Data-dependent Bound for All ERM Algorithms}
\label{sec:erm-results}

Continuing the notation from Section~\ref{sec:PDIS-results}, 
we can also derive a new data-dependent PAC 
bound on the risk of any ERM learning algorithm 
expressed in terms of $\hat{t}_n$, 
which improves over the previous best known 
such bound from \citet*{hanneke:16b} 
(as we discuss in detail in Section~\ref{sec:ERM-proof}).
Specifically, we have the following result.
The arguments underlying the result, and its 
relation to existing results in the literature, 
are discussed in Section~\ref{sec:ERM-proof}.

\begin{theorem}
\label{thm:ERM-coarse}
Denote by $\vc$ the VC dimension of $\H$ \citep*{vapnik:71}.
Let $\PXY$ be any distribution such that 
$\inf_{h \in \H} R_{\PXY}(h) = 0$.
For any $n \in \nats$ and $\delta \in (0,1)$, 
with probability at least $1-\delta$, 
every $h \in V_n$ satisfies
\begin{equation*}
R_{\PXY}(h) = O\!\left( \frac{1}{n} \left( \vc \log\!\left( \frac{\hat{t}_{\lfloor n/2 \rfloor}}{\vc} \right) + \log\!\left(\frac{1}{\delta}\right) \right) \right).
\end{equation*}
\end{theorem}

\subsection{New Generalization Bounds for Compressed Nearest Neighbor Predictors}
\label{sec:nearest-neighbor-summary}

In nearest-neighbor classification,
it is a classic fact that the $1$-NN rule is not Bayes-consistent while $k$-NN is, where $k$ grows appropriately with $n$ (see \citet{hanneke:19a} for a detailed background). A recent line of work
\citep{gkn-ieee18+nips,DBLP:conf/aistats/KontorovichW15} has shown that a margin-regularized $1$-NN can be made Bayes-consistent, provided the margin is chosen via an SRM principle. Furthermore, the generalization bounds provided by this technique are compression-based and fully empirical. This line of research culminated in \citet{hanneke:19a},
where an algorithm called OptiNet
was presented and shown to be 
strongly universally
Bayes-consistent in any metric space where {\em any} learner enjoys this property.

The OptiNet algorithm is described
in detail in \citet{hanneke:19a}. Briefly, for any choice of $\gamma>0$, one constructs a $\gamma$-net on the sample --- that is, a $\gamma$-separated set that is also a $\gamma$-cover.
The following greedy algorithm constructs a 
$\gamma$-net in time $O(n^2)$ in any metric space;
more efficient algorithms are known in doubling spaces
\citep{KL04,DBLP:journals/tit/GottliebKK14+colt}.
Initialize the $\gamma$-net $N_{\gamma}$ as the empty set.
Traverse the datapoints in order, and if the $i$th
point is not covered by the current partial net $N_{\gamma}$, it is
appended to $N_{\gamma}$. It is easily verified that this construction indeed yields a $\gamma$-net, which will serve as the compression set. Furthermore, this compression scheme is stable: if any point not included in $N_{\gamma}$ is omitted from the sample, the construction will yield the same $N_{\gamma}$.

Given the $\gamma$-net $N_{\gamma}$ constructed as above, 
based on a given data set $S$, 
define a classifier $\hat{h}_{\gamma} = \alg_{\gamma}(S)$ 
as follows:
For any point $x \in \X$
let
$\hat{x}$ be the element of $N_{\gamma}$ nearest to 
$x$ in the metric (breaking ties according to 
a measurable total order of the space $\X$; see 
\citealp*{hanneke:19a}), 
and predict $\hat{h}_{\gamma}(x) = \hat{y}$, 
where $\hat{y}$ is the majority label 
among all $(x',y') \in S$ for which 
$\hat{x}$ is also the nearest element to $x'$ 
in the $\gamma$-net $N_{\gamma}$.
For this classifier, we have
the following
result
(whose proof is 
deferred to
Section~\ref{sec:optinet-proof}):

\begin{theorem}
\label{thm:NN-bernstein}
Fix any $\gamma > 0$.
For any distribution $\PXY$, any $n \in \nats$, 
and any $\delta \in (0,1)$, 
for $S \sim \PXY^{n}$, 
with probability at least $1-\delta$, 
the classifier $\hat{h}_{\gamma} = \alg_{\gamma}(S)$ 
satisfies 
\begin{equation*}
\left| R_{\PXY}(\hat{h}_{\gamma}) - \hat{R}_{S}(\hat{h}_{\gamma}) \right| 
= O\!\left( \sqrt{ \hat{R}_{S}(\hat{h}_{\gamma}) \frac{1}{n} \left( |N_{\gamma}| + \log\!\left(\frac{1}{\delta}\right) \right) }
+ \frac{1}{n} \left( |N_{\gamma}| + \log\!\left(\frac{1}{\delta}\right) \right)
\right).
\end{equation*}
\end{theorem}

This refines a data-dependent bound used by 
\citet*{hanneke:19a}, which contained an additional 
log factor, and was based on a significantly 
more-involved argument needed in order to 
maintain permutation-invariance of certain subsets 
of the arguments to the reconstruction function 
used there (since otherwise the log factor 
would be of the form $\log(n)$, rather than 
$\log(n/|N_{\gamma}|)$, which was important 
for the proof of universal consistency in that 
work).  In contrast, since the above bound 
is proven via establishing that $\hat{h}_n$ 
agrees with a \emph{stable} compression scheme, 
we do not need to worry about the fact that 
the corresponding reconstruction function 
may be order-dependent, since there is no 
log factor to be concerned with.
\citet*{hanneke:19a} used their bound on 
$R_{\PXY}(\hat{h}_{\gamma})$ in a procedure, 
called OptiNet, which optimizes the bound 
over the choice of $\gamma$; they show that 
doing so yields a universally strongly 
Bayes-consistent learning algorithm, in 
all metric spaces where Bayes-consistent learning 
is possible.  
Based on the above refinement, we could 
instead substitute this new tighter bound 
\begin{equation*}
\hat{R}_{S}(\hat{h}_{\gamma}) + 
O\!\left( \sqrt{ \hat{R}_{S}(\hat{h}_{\gamma}) \frac{1}{n} \left( |N_{\gamma}| + \log\!\left(\frac{1}{\delta}\right) \right) }
+ \frac{1}{n} \left( |N_{\gamma}| + \log\!\left(\frac{1}{\delta}\right) \right)
\right)
\end{equation*}
in the optimization in OptiNet, 
and the universal consistency result 
would still hold.

\section{Definitions and Theorems for Stable Compression Schemes}

The following notion was introduced by 
\citet*{littlestone:86,floyd:95}.

\begin{definition}
\label{defn:compression-scheme}
A \emph{compression scheme} 
$(\kappa,\rho)$ consists of a 
\emph{compression function} $\kappa$, 
which maps any $S \in (\X \times \Y)^{*}$ 
to a subsequence $\kappa(S) \subseteq S$, 
and a \emph{reconstruction function} 
$\rho : (\X \times \Y)^{*} \to \Y^{\X}$ 
mapping any $S \in (\X \times \Y)^{*}$ to a 
function $\rho(S) : \X \to \Y$.
\end{definition}

For our purposes below, to be a valid compression scheme 
$(\kappa,\rho)$, we also require that the function 
$(S,x) \mapsto \rho(\kappa(S))(x)$ 
mapping $(\X \times \Y)^{n} \times \X \to \Y$
be a measurable function, for every $n \in \nats \cup \{0\}$.
For $k \in \nats \cup \{0\}$, 
we say that a compression scheme $(\kappa,\rho)$ has 
\emph{size} $k$ if every $S \in (\X \times \Y)^*$ has 
$|\kappa(S)| \leq k$.
The following definition is from \citep*{bousquet:20} 
(previously also studied by \citealp*{vapnik:74,zhivotovskiy2017}).

\begin{definition}
\label{defn:stable} 
A compression scheme 
$(\kappa,\rho)$ is called 
\emph{stable} if, for any 
$S \in (\X \times \Y)^{*}$, 
$\forall S' \subseteq S \setminus \kappa(S)$, 
it holds that 
$\rho(\kappa(S \setminus S')) = \rho(\kappa(S))$.
\end{definition}

\subsection{Results for Sample-Consistent Stable Compression Schemes}
\label{sec:realizable-stable}

In the literature on sample compression schemes, 
considerable attention has been given to the 
special case when the compression scheme is 
\emph{sample-consistent}, meaning that for a 
data set $S$, it holds that $\hat{R}_{S}(\rho(\kappa(S))) = 0$.
For general sample compression schemes of size $k$, 
the best general result for the sample-consistent case 
is due to \citet*{littlestone:86,floyd:95}, stating 
that for $S \sim \PXY^n$, with probability at least $1-\delta$, 
if $\hat{R}_{S}(\rho(\kappa(S)))=0$ then 
\begin{equation*}
R_{\PXY}(\rho(\kappa(S))) = O\!\left(\frac{1}{n}\left( k \log(n) + \log\!\left( \frac{1}{\delta} \right) \right) \right),
\end{equation*}
with a slight improvement to 
\begin{equation*}
R_{\PXY}(\rho(\kappa(S))) = O\!\left(\frac{1}{n}\left( k \log\!\left(\frac{n}{k}\right) + \log\!\left( \frac{1}{\delta} \right) \right) \right)
\end{equation*}
in the case that the reconstruction function $\rho$
is permutation-invariant.
\citet*{floyd:95} also showed that the above bounds 
are sharp, in that there exist compression schemes 
for which, for certain distributions, 
one can prove lower bounds matching the above 
inequalities up to numerical constant factors.

However, in the special case of \emph{stable} compression 
schemes, \citet*{bousquet:20} proved that the 
log factor in the above inequalities is superfluous.
Specifically, they established the following theorem.

\begin{theorem}
\label{thm:original-stable-compression}
Let $k \in \nats \cup \{0\}$ and let 
$(\kappa,\rho)$ be any stable compression scheme 
of size $k$.
For any distribution $\PXY$, any integer $n > 2k$, and 
any $\delta \in (0,1)$, 
letting $S \sim \PXY^{n}$, 
with probability at least $1-\delta$, 
if $\hat{R}_{S}(\rho(\kappa(S)))=0$, 
then
\begin{equation*}
R_{\PXY}(\rho(\kappa(S)) 
\leq \frac{2}{n-2k} \left( k \ln(4) + \ln\!\left( \frac{1}{\delta} \right) \right).
\end{equation*}
\end{theorem}

\citet*{bousquet:20} stated several implications of this 
result for compression schemes of known bounded size $k$. 
For instance, their result established, for the first time, 
that the support vector machine 
achieves a (minimax-optimal) generalization bound 
$R_{\PXY}(\hat{h}_{{\rm SVM}}) = O\!\left( \frac{1}{n}\left( d + \log\!\left(\frac{1}{\delta}\right) \right) \right)$ 
for learning linear separators on $\reals^{d}$ in the 
realizable case, based on the fact that it can be 
expressed as a stable compression scheme of size $d+1$.

As one part of the present work, 
we explore further implications of 
this result, focusing on \emph{data-dependent} generalization bounds.
For this purpose, we will use the following easy extension of 
Theorem~\ref{thm:original-stable-compression} holding for 
data-dependent compression set sizes.

\begin{theorem}
\label{thm:realizable}
Let $(\kappa,\rho)$ be any stable compression scheme.
For any distribution $\PXY$, any $n \in \nats$, 
and any $\delta \in (0,1)$, 
for $S \sim \PXY^{n}$, 
with probability at least $1-\delta$, 
if $\hat{R}_{S}(\rho(\kappa(S)))=0$ and 
$|\kappa(S)| < n/2$, 
then
\begin{equation*}
R_{\PXY}(\rho(\kappa(S))) 
\leq \frac{2}{n - 2|\kappa(S)|} \left( |\kappa(S)| \ln(4) + \ln\!\left( \frac{(|\kappa(S)|+1)(|\kappa(S)|+2)}{\delta} \right) \right).
\end{equation*}
\end{theorem}
\begin{proof}
For each $k \in \nats \cup \{0\}$, 
let $(\kappa_{k},\rho_{k})$ be a compression scheme 
such that, for any $S$, if 
$|\kappa(S)|\leq k$, then 
$\kappa_{k}(S)=\kappa(S)$, 
and otherwise $\kappa_{k}(S) = 
\emptyset
$; 
in any case, $\rho_{k} = \rho$.
In particular, note that $|\kappa_{k}(S)| \leq k$ always.
For each $k$, Theorem~\ref{thm:original-stable-compression} 
implies that, with probability at least $1-\frac{\delta}{(k+1)(k+2)}$,
if $\hat{R}_{S}(\rho_{k}(\kappa_{k}(S))) = 0$, then 
\begin{equation*}
R_{\PXY}(\rho_{k}(\kappa_{k}(S))) 
\leq \frac{2}{n-2k}\left( k \ln(4) + \ln\!\left(\frac{(k+1)(k+2)}{\delta}\right) \right).
\end{equation*}
By the union bound, the above claim holds 
simultaneously for all $k \in \nats \cup \{0\}$ 
with probability at least $1 - \sum_{k} \frac{\delta}{(k+1)(k+2)} = 1-\delta$.
Finally, note that there necessarily exists \emph{some} 
$k \in \nats \cup \{0\}$ for which 
$|\kappa(S)|=k$, in which case 
$\rho(\kappa(S)) = \rho_{k}(\kappa_{k}(S))$ 
for this $k$.  The theorem follows immediately from this.
\end{proof}

In particular, the following corollary is immediate, 
by relaxing the expression in the above theorem 
into a simpler form.

\begin{corollary}
\label{cor:realizable-simple}
Let $(\kappa,\rho)$ be any stable compression scheme.
For any distribution $\PXY$, any $n \in \nats$, 
and any $\delta \in (0,1)$,
for $S \sim \PXY^{n}$, 
with probability at least $1-\delta$, 
if $\hat{R}_{S}(\rho(\kappa(S)))=0$
then
\begin{equation*}
R_{\PXY}(\rho(\kappa(S))) 
\leq \frac{4}{n} \left( 6 |\kappa(S)| + \ln\!\left( \frac{e}{\delta} \right) \right).
\end{equation*}
\end{corollary}
\begin{proof}
Since $\ln(x) < \sqrt{x}/2$ for $x \geq 3$, 
we have $\ln\!\left( (|\kappa(S)|+2)^2 \right) 
< (1/2)|\kappa(S)|+1$, so that 
the right hand side of the inequality in Theorem~\ref{thm:realizable} 
is at most 
$\frac{2}{n - 2|\kappa(S)|} \left( |\kappa(S)| ((1/2)+\ln(4)) + \ln\!\left( \frac{e}{\delta} \right) \right)$.
Furthermore, since this is greater than $1$ 
if $|\kappa(S)| > n / (3 + 2\ln(4))$, 
and $R_{\PXY}(\rho(\kappa(S))) \leq 1$ always, 
any upper bound on this expression nondecreasing in $|\kappa(S)|$
and holding for all 
$|\kappa(S)| \leq n / (3 + 2\ln(4))$
is a valid bound on $R_{\PXY}(\rho(\kappa(S)))$.
In particular, for $|\kappa(S)| \leq n / (3 + 2\ln(4))$, 
it holds that 
$n - 2|\kappa(S)| \geq \left( 1 - \frac{2}{3+2\ln(4)} \right) n$, 
which implies a bound of 
$\frac{1}{n} \left( 4 \ln(16 e^3)|\kappa(S)| + \frac{2\ln(16 e^3)}{\ln(16 e)}\ln\!\left( \frac{e}{\delta} \right) \right)$.
The stated bound follows by noting 
$\frac{2\ln(16 e^3)}{\ln(16 e)} < 4$
and 
$4 \ln(16 e^3) < 24$.
\end{proof}

\subsection{Results for Agnostic Stable Compression Schemes}
\label{sec:agnostic-stable}

The best known bounds for the general agnostic setting, 
holding for any compression scheme, are from 
\citet*{graepel:05}.  Specifically, for any compression 
scheme of size $k$, 
they show a bound (holding with probability at least $1-\delta$) 
\begin{equation*}
\left| R_{\PXY}(\rho(\kappa(S))) - \hat{R}_{S}(\rho(\kappa(S))) \right|
= O\!\left( \sqrt{\frac{1}{n} \left( k \log(n) + \log\!\left(\frac{1}{\delta}\right) \right)} \right),
\end{equation*}
or a slightly tighter bound 
\begin{equation*}
\left| R_{\PXY}(\rho(\kappa(S))) - \hat{R}_{S}(\rho(\kappa(S))) \right|
= O\!\left( \sqrt{\frac{1}{n} \left( k \log\!\left(\frac{n}{k}\right) + \log\!\left(\frac{1}{\delta}\right) \right)} \right)
\end{equation*}
in the case of permutation-invariant 
reconstruction function $\rho$.
It was shown by \citet*{hanneke:19c} that 
both of these bounds are generally sharp:
that is, for any $k,n$, there exist compression 
schemes of size $k$, and distributions $\PXY$, 
such that a lower bound holds which matches the 
above up to numerical constants.

Here we show that if the compression scheme is \emph{stable}, 
the log factor in the above bounds can be removed, 
analogous to the result of \citet*{bousquet:20} 
for the realizable case.

\begin{theorem}
\label{thm:agnostic}
For any $k \in \nats \cup \{0\}$, 
let $(\kappa,\rho)$ be any stable compression scheme 
of size $k$.
For any distribution $\PXY$, any $n \in \nats$ with $n > 2k$, 
and any $\delta \in (0,1)$, 
for $S \sim \PXY^{n}$, 
with probability at least $1-\delta$, 
\begin{equation*}
\left| R_{\PXY}(\rho(\kappa(S))) - \hat{R}_{S}(\rho(\kappa(S))) \right|
\leq \sqrt{ \frac{4}{n-2k} \left( k \ln(4) + \ln\!\left(\frac{4}{\delta}\right) \right) }.
\end{equation*}
\end{theorem}
\begin{proof}
For brevity, define $[m] = \{1,\ldots,m\}$ for any $m \in \nats$.
The proof partly follows an argument 
from the original proof of Theorem~\ref{thm:original-stable-compression} 
by \citet{bousquet:20}, but with some important 
modifications to account for the fact that 
$\hat{R}_{S}(\rho(\kappa(S)))$ may be nonzero.

If $k = 0$, the result trivially follows from Hoeffding's 
inequality, so let us suppose $k \geq 1$.
As in the proof of \citet{bousquet:20}, fix any $T_n \in [n-1]$ 
and let $\I_n$ be any family of subsets of $[n]$ 
with the properties that 
each $I \in \I_n$ has $|I| \leq n-T_n$, 
and for every $i_1,\ldots,i_k \in [n]$ 
there exists $I \in \I_n$ with $\{i_1,\ldots,i_k\} \subseteq I$.

In particular, \citet{bousquet:20} construct a family 
$\I_n$ satisfying the properties above with $T_n =k\lfloor n/(2k) \rfloor$, and with $|\I_n| = \binom{2k}{k} < 4^k$: 
namely, let $D_1,\ldots,D_{2k}$ be any partition 
of $[n]$ with each 
$|D_i| \in \{ \lfloor n/(2k)\rfloor, \lceil n/(2k) \rceil \}$, 
and define $\I_n = \left\{ \bigcup \{ D_j : j \in \J \} : \J \subseteq [2k], |\J|=k \right\}$;  
that is, $\I_n$ contains all unions of exactly $k$ 
of the $2k$ sets $D_j$.

Let $S = \{(X_i,Y_i)\}_{i=1}^{n} \sim \PXY^n$, 
and for any $I \subseteq [n]$ 
define $S_I = \{(X_i,Y_i) : i \in I \}$.
As a new component needed in the present proof, 
let $\sigma : [n] \to [n]$ be a uniform random permutation 
of $[n]$; this will only become important in the second 
half of the proof below.

For any $I \subset [n]$, since $S_{[n]\setminus I}$ 
is independent of $S_{I}$, 
Hoeffding's inequality (applied under the conditional 
distribution given $S_{I}$) 
and the law of total probability 
imply that, with probability at least 
$1 - \frac{\delta}{2|\I_n|}$, 
\begin{equation*}
\left| R_{\PXY}(\rho(\kappa(S_{I}))) - \hat{R}_{S_{[n] \setminus I}}(\rho(\kappa(S_{I}))) \right| 
\leq \sqrt{ \frac{ \ln\!\left(4|\I_n|/\delta\right) }{2 (n-|I|)} }.
\end{equation*}
Applying this under the conditional distribution given $\sigma$,
together with the union bound and the law of total probability,
we have that with probability at least $1-\frac{\delta}{2}$, 
every $I \in \I_{n}$ has 
\begin{equation*}
\left| R_{\PXY}(\rho(\kappa(S_{\sigma^{-1}(I)}))) - \hat{R}_{S_{[n] \setminus \sigma^{-1}(I)}}(\rho(\kappa(S_{\sigma^{-1}(I)}))) \right| 
\leq \sqrt{ \frac{ \ln\!\left(4|\I_n|/\delta\right) }{2 (n-|I|)} }.
\end{equation*}

In particular, letting $i_1^*,\ldots,i_{|\kappa(S)|}^*$ 
be the $|\kappa(S)|$ indices such that 
$\kappa(S) = \{(X_{i_j^*},Y_{i_j^*})\}_{j=1}^{|\kappa(S)|}$, 
the defining properties of $\I_n$ imply that 
there exists $I^* \in \I_n$ with 
$\{ \sigma(i_1^*),\ldots,\sigma(i_{|\kappa(S)|}^*)\} \subseteq I^*$.
Since $(\kappa,\rho)$ is a \emph{stable} compression scheme, 
this also implies 
$\rho(\kappa(S_{\sigma^{-1}(I^*)}))=\rho(\kappa(S))$.
Furthermore, by the defining properties of $\I_n$, 
we have $n - |I^*| \geq T_n$.
Therefore, on the above event of probability at least $1-\frac{\delta}{2}$, 
\begin{equation}
\label{eqn:agnostic-compression-1}
\left| R_{\PXY}(\rho(\kappa(S))) - \hat{R}_{S_{[n] \setminus \sigma^{-1}(I^*)}}(\rho(\kappa(S))) \right| 
\leq \sqrt{ \frac{ \ln\!\left(4|\I_n|/\delta\right) }{2 T_n} }.
\end{equation}

Now, unlike the original proof of \citet{bousquet:20}, 
to complete the present proof we must still relate 
$\hat{R}_{S_{[n] \setminus \sigma^{-1}(I^*)}}(\rho(\kappa(S)))$ 
to $\hat{R}_{S}(\rho(\kappa(S)))$.
This is where the random permutation $\sigma$ becomes 
important, as it enables us to introduce a concentration 
argument which accounts for the possibility that 
$\rho(\kappa(\cdot))$ may be order-dependent in its argument.
Let $\hat{h} = \rho(\kappa(S))$.
For each $i \in [n]$, 
let $\ell_i = \ind[ \hat{h}(X_i) \neq Y_i ]$.
For any $I \in \I_n$, 
by Hoeffding's inequality (for sampling without replacement; 
see \citealp*{hoeffding:63}) applied under the conditional 
distribution given $S$, together with the law of total probability, 
with probability at least $1 - \frac{\delta}{2|\I_n|}$, 
it holds that 
\begin{equation*}
\left| \frac{1}{n-|I|}\sum_{i \in [n] \setminus \sigma^{-1}(I)} r_i 
- \hat{R}_{S}(\rho(\kappa(S))) \right| 
\leq \sqrt{ \frac{\ln\!\left( 4|\I_n| / \delta \right)}{2 (n-|I|)}}.
\end{equation*}
By the union bound, this holds simultaneously for all $I \in \I_n$
with probability at least $1 - \frac{\delta}{2}$.
In particular, taking $I = I^*$, and recalling that 
$n-|I^*| \geq T_n$, on this event we have that
\begin{equation}
\label{eqn:agnostic-compression-2}
\left| \hat{R}_{S_{[n] \setminus \sigma^{-1}(I^*)}}(\rho(\kappa(S))) 
- \hat{R}_{S}(\rho(\kappa(S))) 
\right| 
\leq \sqrt{ \frac{ \ln\!\left(4|\I_n|/\delta\right) }{2 T_n} }.
\end{equation}

By the union bound, the above two events 
(each of probability at least $1-\frac{\delta}{2}$) 
hold simultaneously with probability at least $1-\delta$, 
in which case 
\eqref{eqn:agnostic-compression-1} 
and \eqref{eqn:agnostic-compression-2} 
together imply 
\begin{align*}
& \left| R_{\PXY}(\rho(\kappa(S))) - \hat{R}_{S}(\rho(\kappa(S))) \right| 
\\ & \leq 
\left| R_{\PXY}(\rho(\kappa(S))) - \hat{R}_{S_{[n] \setminus \sigma^{-1}(I^*)}}(\rho(\kappa(S))) \right| 
+ \left| \hat{R}_{S_{[n] \setminus \sigma^{-1}(I^*)}}(\rho(\kappa(S))) - \hat{R}_{S}(\rho(\kappa(S))) \right| 
\\ & \leq \sqrt{ \frac{ 2 \ln\!\left(4|\I_n|/\delta\right) }{T_n} }.
\end{align*}

The theorem now immediately follows by 
plugging in the aforementioned 
family $\I_n$ from \citet{bousquet:20}, 
having $|\I_n| = \binom{2k}{k} < 4^k$ 
and $T_n = k \lfloor n/(2k) \rfloor > \frac{n-2k}{2}$.
\end{proof}

As above, this easily extends to data-dependent 
compression sizes, as stated in the following theorem.
The proof follows the same argument as in the proof of 
Theorem~\ref{thm:realizable}, and so we omit the details.

\begin{theorem}
\label{thm:agnostic-data-dependent}
Let $(\kappa,\rho)$ be any stable compression scheme.
For any distribution $\PXY$, any $n \in \nats$, 
and any $\delta \in (0,1)$, 
for $S \sim \PXY^{n}$, 
with probability at least $1-\delta$, 
if $|\kappa(S)| < n/2$, 
then
\begin{equation*}
\left| R_{\PXY}(\rho(\kappa(S))) - \hat{R}_{S}(\rho(\kappa(S))) \right|
\leq \sqrt{ \frac{4}{n-2|\kappa(S)|} \left( |\kappa(S)| \ln(4) + \ln\!\left(\frac{4 (|\kappa(S)|+1)(|\kappa(S)|+2)}{\delta}\right) \right) }.
\end{equation*}
\end{theorem}

Also analogous to the results for the realizable case, 
the above bound can be further relaxed into a simple expression, 
as follows.  The proof is nearly identical to that of 
Corollary~\ref{cor:realizable-simple}, and so we omit it 
for brevity.

\begin{corollary}
\label{cor:agnostic-simple}
Let $(\kappa,\rho)$ be any stable compression scheme.
For any distribution $\PXY$, any $n \in \nats$, 
and any $\delta \in (0,1)$, 
for $S \sim \PXY^{n}$, 
with probability at least $1-\delta$, 
then
\begin{equation*}
\left| R_{\PXY}(\rho(\kappa(S))) - \hat{R}_{S}(\rho(\kappa(S))) \right|
\leq \sqrt{ \frac{8}{n} \left( 6 |\kappa(S)| + \ln\!\left(\frac{4 e}{\delta}\right) \right) }.
\end{equation*}
\end{corollary}

While the bound of Theorem~\ref{thm:agnostic} 
holds for $1-\delta$ fraction of data sets from 
any distribution, and is therefore more general 
than Theorem~\ref{thm:original-stable-compression} 
(which restricts to the sample-consistent case), 
the bound is not as tight in the specific case where 
Theorem~\ref{thm:original-stable-compression} applies.
As such, it is desirable to also state a bound which 
interpolates between the two: that is, which does 
not require the compression scheme to be sample-consistent 
to provide a non-trivial bound, 
but yet is able to recover the form of the 
bound in Theorem~\ref{thm:original-stable-compression} 
in the case where it happens to be sample-consistent.
We provide such a result in the following theorem.

\begin{theorem}
\label{thm:bernstein}
For any $k \in \nats \cup \{0\}$, 
let $(\kappa,\rho)$ be any stable compression scheme 
of size $k$.
For any distribution $\PXY$, any $n \in \nats$ 
with $n > 4k$, 
and any $\delta \in (0,1)$, 
for $S \sim \PXY^{n}$, 
with probability at least $1-\delta$, 
\begin{align*}
& \left| R_{\PXY}(\rho(\kappa(S))) - \hat{R}_{S}(\rho(\kappa(S))) \right|
\\ & \leq \sqrt{\hat{R}_{S}(\rho(\kappa(S)))\frac{72}{n} \left( k \ln(4) + \ln\!\left( \frac{4}{\delta} \right)\right)}
+ \frac{32}{n} \left( k \ln(4) + \ln\!\left( \frac{4}{\delta} \right)\right).
\end{align*}
\end{theorem}

Before stating the proof, we first recall the 
following so-called ``ratio-type'' inequality, 
based on Bernstein's inequality.

\begin{lemma}
\label{lem:bernstein}
For any $n \in \nats$, we consider two cases simultaneously: 
(i) let $p \in [0,1]$ and let $Z_1,\ldots,Z_n$ 
be i.i.d.\ ${\rm Bernoulli}(p)$ random variables,
(ii) let $t \geq n$, $\{B_1,\ldots,B_t\} \in \{0,1\}^t$, 
$p = \frac{1}{t}\sum_{i=1}^{t} B_i$, 
and let $Z_1,\ldots,Z_n$ be random variables 
sampled uniformly without replacement 
from $\{B_1,\ldots,B_t\}$.
In either case, for any $\delta \in (0,1)$, 
defining $\bar{Z} = \frac{1}{n} \sum_{i=1}^{n} Z_i$, 
with probability at least $1-\delta$, 
\begin{equation*}
\left| \bar{Z} - p \right| 
\leq \sqrt{ \min\!\left\{ 2\bar{Z}, p \right\} \frac{2}{n} \ln\!\left(\frac{2}{\delta}\right) }
+ \frac{4}{n} \ln\!\left(\frac{2}{\delta}\right). 
\end{equation*}
\end{lemma}
\begin{proof}
In both cases covered by the claim, 
Bernstein's inequality implies that
\begin{equation*}
\P\!\left( \left| \bar{Z} - p \right| > \eps \right) 
\leq 2 \exp\!\left\{ - \frac{(1/2) \eps^2 n}{p + (\eps/3)} \right\}.
\end{equation*}
Setting the right hand side equal $\delta$ and solving for $\eps$ 
yields that, with probability at least $1-\delta$, 
\begin{equation*}
\left| \bar{Z} - p \right| 
\leq \sqrt{p \frac{2}{n}\ln\!\left(\frac{2}{\delta}\right) + \frac{1}{9 n^2} \ln^2\!\left(\frac{2}{\delta}\right) } 
+ \frac{1}{3n}\ln\!\left(\frac{2}{\delta}\right).
\end{equation*}
In particular, relaxing the right hand side above 
implies that, on this event, 
\begin{equation}
\label{eqn:bernstein-basic-1}
\left| \bar{Z} - p \right| 
\leq \sqrt{p \frac{2}{n}\ln\!\left(\frac{2}{\delta}\right)} 
+ \frac{2}{3n}\ln\!\left(\frac{2}{\delta}\right).
\end{equation}
Furthermore, for any non-negative values 
$A,B,C$, it holds that 
$A \leq B + C \sqrt{A} \Rightarrow A \leq B + C^2 + \sqrt{B}C$.
Therefore, on the above event, 
\begin{align*}
p & 
\leq \bar{Z} + \frac{8}{3n}\ln\!\left(\frac{2}{\delta}\right) 
+ \sqrt{\bar{Z} + \frac{2}{3n}\ln\!\left(\frac{2}{\delta}\right)}\sqrt{\frac{2}{n}\ln\!\left(\frac{2}{\delta}\right)}
\leq 
2 \bar{Z} + \frac{16}{3n}\ln\!\left(\frac{2}{\delta}\right).
\end{align*}
Plugging back into \eqref{eqn:bernstein-basic-1} 
yields that, on this same event, 
\begin{align*}
\left| \bar{Z} - p \right| 
& \leq \sqrt{2 \bar{Z} \frac{2}{n}\ln\!\left(\frac{2}{\delta}\right)} 
+ \left(\sqrt{\frac{8}{3}}+\frac{1}{3}\right) \frac{2}{n}\ln\!\left(\frac{2}{\delta}\right)
\leq \sqrt{2 \bar{Z} \frac{2}{n}\ln\!\left(\frac{2}{\delta}\right)} 
+ \frac{4}{n}\ln\!\left(\frac{2}{\delta}\right).
\end{align*}
This inequality and \eqref{eqn:bernstein-basic-1} 
together imply the claimed bound.
\end{proof}

We are now ready for the proof of Theorem~\ref{thm:bernstein}.

\begin{proof}[of Theorem~\ref{thm:bernstein}]
This proof follows essentially similar arguments to 
the proof of Theorem~\ref{thm:agnostic}, except using 
Lemma~\ref{lem:bernstein} in place of Hoeffding's inequality 
in both places in the proof where such inequalities 
are used.
Let $\I_n$ and $T_n$ be as in the proof of 
Theorem~\ref{thm:agnostic}, 
and let $[m] = \{1,\ldots,m\}$ for any $m \in \nats$.

If $k = 0$, the result trivially follows from 
Lemma~\ref{lem:bernstein}, so let us suppose $k \geq 1$.
Let $S = \{(X_i,Y_i)\}_{i=1}^{n} \sim \PXY^n$, 
and for any $I \subseteq [n]$ 
define $S_I = \{(X_i,Y_i) : i \in I \}$.
As in Theorem~\ref{thm:agnostic}, 
let $\sigma : [n] \to [n]$ be a uniform random permutation 
of $[n]$.

For any $I \subset [n]$, since $S_{[n]\setminus I}$ 
is independent of $S_{I}$, 
Lemma~\ref{lem:bernstein} (applied under the conditional 
distribution given $S_{I}$) 
and the law of total probability 
imply that, with probability at least 
$1 - \frac{\delta}{2|\I_n|}$, 
\begin{align*}
& \left| R_{\PXY}(\rho(\kappa(S_{I}))) - \hat{R}_{S_{[n] \setminus I}}(\rho(\kappa(S_{I}))) \right| 
\leq \sqrt{ \hat{R}_{S_{[n] \setminus I}}(\rho(\kappa(S_{I})))\frac{4}{n-|I|} \ln\!\left(\frac{4|\I_n|}{\delta}\right) }
+ \frac{4}{n-|I|} \ln\!\left(\frac{4|\I_n|}{\delta}\right).
\end{align*}
Applying this under the conditional distribution given $\sigma$,
together with the union bound and the law of total probability,
we have that with probability at least $1-\frac{\delta}{2}$, 
every $I \in \I_{n}$ has 
\begin{align*}
& \left| R_{\PXY}(\rho(\kappa(S_{\sigma^{-1}(I)}))) - \hat{R}_{S_{[n] \setminus \sigma^{-1}(I)}}(\rho(\kappa(S_{\sigma^{-1}(I)}))) \right| 
\\ & 
\leq \sqrt{ \hat{R}_{S_{[n] \setminus \sigma^{-1}(I)}}(\rho(\kappa(S_{\sigma^{-1}(I)})))\frac{4}{n-|I|} \ln\!\left(\frac{4|\I_n|}{\delta}\right) }
+ \frac{4}{n-|I|} \ln\!\left(\frac{4|\I_n|}{\delta}\right).
\end{align*}

In particular, letting $i_1^*,\ldots,i_{|\kappa(S)|}^*$ 
be the $|\kappa(S)|$ indices such that 
$\kappa(S) = \{(X_{i_j^*},Y_{i_j^*})\}_{j=1}^{|\kappa(S)|}$, 
the defining properties of $\I_n$ imply that 
there exists $I^* \in \I_n$ with 
$\{ \sigma(i_1^*),\ldots,\sigma(i_{|\kappa(S)|}^*)\} \subseteq I^*$.
Since $(\kappa,\rho)$ is a \emph{stable} compression scheme, 
this also implies 
$\rho(\kappa(S_{\sigma^{-1}(I^*)}))=\rho(\kappa(S))$.
Furthermore, by the defining properties of $\I_n$, 
we have $n - |I^*| \geq T_n$.
Also note that 
$\hat{R}_{S_{[n] \setminus \sigma^{-1}(I^*)}}(\rho(\kappa(S)))
\leq \frac{n}{T_n} \hat{R}_{S}(\rho(\kappa(S)))$.
Therefore, on the above event of probability at least $1-\frac{\delta}{2}$, 
\begin{align}
\left| R_{\PXY}(\rho(\kappa(S))) - \hat{R}_{S_{[n] \setminus \sigma^{-1}(I^*)}}(\rho(\kappa(S))) \right| 
\leq 
\sqrt{ \frac{n}{T_n} \hat{R}_{S}(\rho(\kappa(S)))\frac{4}{T_n} \ln\!\left(\frac{4|\I_n|}{\delta}\right) }
+ \frac{4}{T_n} \ln\!\left(\frac{4|\I_n|}{\delta}\right).
\label{eqn:bernstein-compression-1}
\end{align}

Let $\hat{h} = \rho(\kappa(S))$.
For each $i \in [n]$, 
let $\ell_i = \ind[ \hat{h}(X_i) \neq Y_i ]$.
For any $I \in \I_n$, 
by Lemma~\ref{lem:bernstein} (the case holding for 
sampling without replacement) applied under the conditional 
distribution given $S$, together with the law of total probability, 
with probability at least $1 - \frac{\delta}{2|\I_n|}$, 
it holds that 
\begin{align*}
\left| \frac{1}{n-|I|}\sum_{i \in [n] \setminus \sigma^{-1}(I)} r_i 
- \hat{R}_{S}(\rho(\kappa(S))) \right| 
\leq 
\sqrt{ \hat{R}_{S}(\rho(\kappa(S)))  \frac{2}{n-|I|}\ln\!\left(\frac{4|\I_n|}{\delta}\right) }
+ \frac{4}{n-|I|}\ln\!\left(\frac{4|\I_n|}{\delta}\right).
\end{align*}
By the union bound, this holds simultaneously for all $I \in \I_n$
with probability at least $1 - \frac{\delta}{2}$.
In particular, taking $I = I^*$, and recalling that 
$n-|I^*| \geq T_n$, on this event we have that
\begin{align}
\left| \hat{R}_{S_{[n] \setminus \sigma^{-1}(I^*)}}(\rho(\kappa(S))) 
- \hat{R}_{S}(\rho(\kappa(S))) 
\right| 
\leq 
\sqrt{ \hat{R}_{S}(\rho(\kappa(S)))  \frac{2}{T_n}\ln\!\left(\frac{4|\I_n|}{\delta}\right) }
+ \frac{4}{T_n}\ln\!\left(\frac{4|\I_n|}{\delta}\right).
\label{eqn:bernstein-compression-2}
\end{align}

By the union bound, the above two events 
(each of probability at least $1-\frac{\delta}{2}$) 
hold simultaneously with probability at least $1-\delta$, 
in which case 
\eqref{eqn:bernstein-compression-1} 
and \eqref{eqn:bernstein-compression-2} 
together imply 
\begin{align*}
& \left| R_{\PXY}(\rho(\kappa(S))) - \hat{R}_{S}(\rho(\kappa(S))) \right| 
\\ & \leq 
\left| R_{\PXY}(\rho(\kappa(S))) - \hat{R}_{S_{[n] \setminus \sigma^{-1}(I^*)}}(\rho(\kappa(S))) \right| 
+ \left| \hat{R}_{S_{[n] \setminus \sigma^{-1}(I^*)}}(\rho(\kappa(S))) - \hat{R}_{S}(\rho(\kappa(S))) \right| 
\\ & \leq 
\left( 1 + \sqrt{\frac{2n}{T_n}}\right) \sqrt{ \hat{R}_{S}(\rho(\kappa(S)))  \frac{2}{T_n}\ln\!\left(\frac{4|\I_n|}{\delta}\right) }
+ \frac{8}{T_n}\ln\!\left(\frac{4|\I_n|}{\delta}\right).
\end{align*}

The theorem now immediately follows by 
plugging in the family $\I_n$ from \citet*{bousquet:20} 
(described in the proof of Theorem~\ref{thm:agnostic} above), 
having $|\I_n| = \binom{2k}{k} < 4^k$ 
and $T_n = k \lfloor n/(2k) \rfloor > \frac{n-2k}{2} > \frac{n}{4}$, 
and relaxing the numerical constants to simplify 
the expression.
\end{proof}

As above, we can also easily extend this to 
data-dependent compression sizes, stated in the 
following theorem.  The proof is nearly identical 
to the proof of Theorem~\ref{thm:realizable} 
(except using Theorem~\ref{thm:bernstein} in place 
of Theorem~\ref{thm:original-stable-compression}) 
and so we omit the proof for brevity.

\begin{theorem}
\label{thm:bernstein-data-dependent}
Let $(\kappa,\rho)$ be any stable compression scheme.
For any distribution $\PXY$, any $n \in \nats$, 
and any $\delta \in (0,1)$, 
for $S \sim \PXY^{n}$, 
with probability at least $1-\delta$, 
if $|\kappa(S)| < n/4$, 
then
\begin{align*}
& \left| R_{\PXY}(\rho(\kappa(S))) - \hat{R}_{S}(\rho(\kappa(S))) \right|
\\ & \leq \sqrt{\hat{R}_{S}(\rho(\kappa(S)))\frac{72}{n} \left( |\kappa(S)| \ln(4) + \ln\!\left( \frac{4 (|\kappa(S)|+1)(|\kappa(S)|+2)}{\delta} \right)\right)}
\\ & ~~+ \frac{32}{n} \left( |\kappa(S)| \ln(4) + \ln\!\left( \frac{4(|\kappa(S)|+1)(|\kappa(S)|+2)}{\delta} \right)\right).
\end{align*}
\end{theorem}

Also as above, we can state a bound in a simpler form 
by relaxing the above inequality, as stated in the 
following corollary.  The proof follows similar arguments 
as in the proof of Corollary~\ref{cor:realizable-simple}, so we 
omit the proof for brevity.

\begin{corollary}
\label{cor:bernstein-simple}
Let $(\kappa,\rho)$ be any stable compression scheme.
For any distribution $\PXY$, any $n \in \nats$, 
and any $\delta \in (0,1)$, 
for $S \sim \PXY^{n}$, 
with probability at least $1-\delta$, 
\begin{align*}
& \left| R_{\PXY}(\rho(\kappa(S))) - \hat{R}_{S}(\rho(\kappa(S))) \right| 
\\ & \leq \sqrt{\hat{R}_{S}(\rho(\kappa(S)))\frac{72}{n} \left( 2 |\kappa(S)| + \ln\!\left( \frac{4 e}{\delta} \right)\right)}
+ \frac{32}{n} \left( 2 |\kappa(S)| + \ln\!\left( \frac{4 e}{\delta} \right)\right).
\end{align*}
\end{corollary}

\section{Details of the Applications}
\label{sec:application-proofs}

This section provides the proofs and discussions
related to the various main results from 
Section~\ref{sec:main-results}.

\subsection{Proof of the Optimal PAC Margin Bound for SVM}
\label{sec:svm-margin-bound-proof}

For the SVM algorithm, 
under linearly separable distributions $\PXY$, 
an in-expectation margin bound
was established very early by \citet*{vapnik:74,vapnik:00}: 
namely, 
$\E\!\left[ R_{\PXY}({\rm SVM}(S_n)) \right] 
\leq \E\!\left[ \frac{r(S_{n+1})^2}{\gamma(S_{n+1})^2}\frac{1}{n+1} \right]$, for $S_n \sim \PXY^n$ and $S_{n+1} \sim \PXY^{n+1}$.
However, determining whether SVM obtains 
the optimal data-dependent \emph{PAC} margin bound has remained a challenging 
open problem, with several sub-optimal bounds appearing 
in prior works in the literature, which include extra 
log factors \citep*{shawe-taylor:98,hanneke:19c}.  
We resolve this question here.
Specifically, we prove the following result, 
from which Theorem~\ref{thm:svm-margin} immediately follows.

\begin{theorem}
\label{thm:svm-full}
For any distribution $\PXY$, any $n \in \nats$, 
and any $\delta \in (0,1)$, for $S \sim \PXY^n$, 
with probability at least $1-\delta$, 
if $S$ is linearly separable, then 
letting $\hat{h}_{{\rm SVM}} = {\rm SVM}(S)$, 
$r = r(S)$, and $\gamma = \gamma(S)$, we have
\begin{equation*}
R_{\PXY}(\hat{h}_{{\rm SVM}}) 
\leq \frac{4}{n} \left( 6 \frac{r^2}{\gamma^2} + \ln\!\left(\frac{e}{\delta}\right) \right)
\end{equation*}
and if $\frac{r^2}{\gamma^2} < n/2$, then 
\begin{equation*}
R_{\PXY}(\hat{h}_{{\rm SVM}}) 
\leq \frac{2}{n-2 r^2/\gamma^2} \left( \frac{r^2}{\gamma^2}\ln(4) + 2\ln\!\left( \frac{r^2}{\gamma^2} + 2\right) + \ln\!\left(\frac{1}{\delta}\right) \right).
\end{equation*}
\end{theorem}
\begin{proof}
It has been known since the initial work of 
\citet*{vapnik:74} that SVM can be expressed 
as a compression scheme, where the compression 
points are the \emph{support vectors}: that is, 
the samples receiving non-zero weight in the 
solution to the dual formulation of the optimization 
problem.  However, the support vectors are not 
always uniquely defined, so that the size of the 
compression scheme depends on which solution 
is used.  However, since the actual classifier 
$\hat{h}_{{\rm SVM}}$ does \emph{not} depend on 
which solution we choose, we can analyze 
$R_{\PXY}(\hat{h}_{{\rm SVM}})$ by identifying 
\emph{any} complete set of support vectors 
of some desired number.

Following \citet*{vapnik:74}, in a given data set $S$, 
define an \emph{essential support vector} 
as any $(x,y) \in S$ such that 
${\rm SVM}(S \setminus \{(x,y)\}) \neq {\rm SVM}(S)$.
The essential support vectors do not necessarily 
form a complete set of support vectors (indeed, there 
may be \emph{no} essential support vectors in some 
cases).  However, we can use a universal bound on the 
number of essential support vectors to identify 
a particular compression scheme of a desirable size, 
corresponding to ${\rm SVM}(S)$.
Specifically, \citet*{vapnik:74} showed that for any 
linearly separable data set $S$, there are at most 
$\frac{r(S)^2}{\gamma(S)^2}$ essential support vectors
(see also \citealp*{hanneke:19b}).

Now we describe a compression scheme 
$(\kappa,\rho)$ with 
$|\kappa(S)| \leq \frac{r(S)^2}{\gamma(S)^2}$ 
for linearly separable data sets $S$, such that 
$\rho(\kappa(S)) = {\rm SVM}(S)$.
Fix any $r,\gamma > 0$; we inductively construct a 
compression scheme $(\kappa_{r,\gamma},\rho_{r,\gamma})$ that, 
for any data set $S$ with $r(S) \leq r$ 
and $\gamma(S) \geq \gamma$, it holds that 
$|\kappa_{r,\gamma}(S)| \leq \frac{r^2}{\gamma^2}$
and $\rho_{r,\gamma}(\kappa_{r,\gamma}(S)) = {\rm SVM}(S)$.
In particular, we will always define $\rho_{r,\gamma}(S) = {\rm SVM}(S)$, so that it remains only to define $\kappa_{r,\gamma}$.
First, if $|S| \leq \frac{r^2}{\gamma^2}$, 
simply define $\kappa_{r,\gamma}(S) = S$, 
so that $\rho_{r,\gamma}(\kappa_{r,\gamma}(S))={\rm SVM}(S)$ 
trivially.
This is our base case in the inductive construction.
Next, take as an inductive hypothesis that $S$ is a 
linearly separable set with $r(S) \leq r$, $\gamma(S) \geq \gamma$, and $|S| > \frac{r^2}{\gamma^2}$, 
and that every strict subset $S' \subset S$ has 
$\kappa_{r,\gamma}(S')$
defined, with $|\kappa_{r,\gamma}(S')| \leq \frac{r^2}{\gamma^2}$ 
and $\rho_{r,\gamma}(\kappa_{r,\gamma}(S')) = {\rm SVM}(S')$ 
(noting that $r(S') \leq r(S) \leq r$, and $\gamma(S') \geq \gamma(S) \geq \gamma$).
Since $|S| > \frac{r^2}{\gamma^2} \geq \frac{r(S)^2}{\gamma(S)^2}$,
the result above implies that $S$ necessarily contains at 
least one point that is \emph{not} an essential support 
vector (with respect to applying SVM to $S$). 
Let $(x,y)$ be the first element of $S$ 
(by their order in the sequence $S$) that is not an 
essential support vector, and define 
$\kappa_{r,\gamma}(S) = \kappa_{r,\gamma}(S \setminus \{(x,y)\})$.
By the inductive hypothesis, 
$|\kappa_{r,\gamma}(S)| = 
|\kappa_{r,\gamma}(S \setminus \{(x,y)\})| 
\leq \frac{r^2}{\gamma^2}$, 
and $\rho_{r,\gamma}(\kappa_{r,\gamma}(S)) = 
\rho_{r,\gamma}(\kappa_{r,\gamma}(S \setminus \{(x,y)\}))
= {\rm SVM}(S \setminus \{(x,y)\})$; 
moreover, since $(x,y)$ is not an essential support vector, 
${\rm SVM}(S \setminus \{(x,y)\}) = {\rm SVM}(S)$, 
so that we have confirmed that $\rho_{r,\gamma}(\kappa_{r,\gamma}(S)) = {\rm SVM}(S)$.
By the principle of induction, we have constructed 
$(\kappa_{r,\gamma},\rho_{r,\gamma})$ satisfying the 
claim for all linearly separable $S$ 
with $r(S) \leq r$ and $\gamma(S) \geq \gamma$.

Now, for any linearly separable data set $S$, 
define $\kappa(S) = \kappa_{r(S),\gamma(S)}(S)$, 
and generally define $\rho(\kappa(S)) = {\rm SVM}(\kappa(S))$.
By the above argument, every linearly separable 
set $S$ has 
$|\kappa(S)| \leq \frac{r(S)^2}{\gamma(S)^2}$
and $\rho(\kappa(S)) = {\rm SVM}(S)$.
Moreover, since ${\rm SVM}(\kappa(S)) = {\rm SVM}(S)$, 
any subset $S' \subset S$ with $\kappa(S) \subseteq S'$ 
must also have ${\rm SVM}(S') = {\rm SVM}(S)$ 
(since $S'$ contains a complete set of support vectors 
with respect to applying SVM to $S$), 
so that the property of the construction of $\kappa$ above 
implies $\rho(\kappa(S')) = {\rm SVM}(\kappa(S')) = {\rm SVM}(S') = {\rm SVM}(S) = {\rm SVM}(\kappa(S)) = \rho(\kappa(S))$.
Thus, $(\kappa,\rho)$ is also a \emph{stable} compression scheme.
The theorem now follows immediately from 
Theorem~\ref{thm:realizable} and 
Corollary~\ref{cor:realizable-simple}.
\end{proof}

\begin{remark}
\label{rem:span}
We also note that this bound can be further refined 
by replacing $r(S)$ with the \emph{span} of the data, 
defined by \citet*{vapnik:00}, as that work also 
established a bound on the number of essential 
support vectors in terms of the span of $S$,
and the span is nonincreasing as we inductively 
remove data from $S$ in the argument used in the 
above proof.
\end{remark}

\subsection{A Data-dependent Online-to-Batch Conversion Bound}
\label{sec:online-to-batch}

A result in \citet*{bousquet:20b} establishes a bound 
for online-to-batch conversion for conservative online 
learners with an \textit{a priori} mistake bound.  
Specifically, from \citep*{littlestone:88}, an online 
learning algorithm $\alg$ is a (measurable) 
map $(\X \times \Y)^* \times \X \to \Y$.
For a given concept class $\H \subseteq \Y^{\X}$ 
of functions, the \emph{mistake bound} 
of $\alg$ is defined as 
\begin{equation*}
M(\alg,\H) = \sup_{x_{1},x_{2},\ldots \in \X} \sup_{h \in \H} \sum_{t=1}^{\infty} \ind[ \alg(\{(x_i,h(x_i))\}_{i=1}^{t-1},x_t) \neq h(x_t) ].
\end{equation*}
In other words, $M(\alg,\H)$ is the largest number of 
mistakes the algorithm $\alg$ will make on any sequence 
labeled according to some target concept in $\H$.
It is known that the minimum possible value of $M(\alg,\H)$ is 
equal to the \emph{Littlestone dimension} of $\H$, 
defined by \citet*{littlestone:88}.
As a special type of algorithm of considerable interest, 
an online learning algorithm $\alg$ is called 
\emph{conservative} if the consecutive predictors 
$\alg(\{(x_i,y_i)\}_{i=1}^{t-1},\cdot)$
and 
$\alg(\{(x_i,y_i)\}_{i=1}^{t},\cdot)$
only differ when 
$\alg(\{(x_i,y_i)\}_{i=1}^{t-1},x_t) \neq y_t$: 
that is, the algorithm's hypothesis is only 
updated after each \emph{mistake}.
Formally, $\alg$ is conservative if, 
for any $n$ and 
$(x_1,y_1),\ldots,(x_n,y_n) \in \X \times \Y$, 
letting $\hat{m} = \sum_{t=1}^{n} \ind[ \alg(\{(x_i,y_i)\}_{i=1}^{t-1},x_t) \neq y_t]$ 
and denoting by $i_1,\ldots,i_{\hat{m}}$ the 
subsequence of $\{1,\ldots,n\}$ with 
$\alg(\{(x_i,y_i)\}_{i=1}^{i_j-1},x_{i_j}) \neq y_{i_j}$, 
and letting $i_0 = 0$ and $i_{\hat{m}+1}=n+1$, 
for every $j \in \{0,\ldots,\hat{m}\}$ 
and every $t \in \{i_{j}+1,\ldots,\min\{i_{j+1},n\}\}$, 
$\alg(\{(x_i,y_i)\}_{i=1}^{t-1},x_t) = \alg(\{(x_{i_{j'}},y_{i_{j'}})\}_{j'=1}^{j},x_t)$.

\citet*{bousquet:20b} propose a new PAC bound for 
conservative online learning algorithms.  Specifically, 
for any given data set $S = \{(X_1,Y_1),\ldots,(X_n,Y_n)\}$, 
and any conservative online learning algorithm, they 
consider running $\alg$ through the data set $S$ 
in order, and cycling through repeatedly until it 
makes a full pass through $S$ without making any mistakes.
Formally, letting $t_{i} = i - n\lfloor i/n \rfloor$ 
for each $i \in \nats$, 
define $\hat{h}_n(\cdot) = \alg(\{(X_{t_i},Y_{t_i})\}_{i=1}^{T},\cdot)$
for $T$ the smallest positive integer multiple of $n$ 
for which 
$\sum_{j=T-n}^{T-1} \ind[ \alg(\{(X_{t_i},Y_{t_i})\}_{i=1}^{j},X_{t_{j+1}}) \neq Y_{t_{j+1}} ] = 0$.
If no such $T$ exists, 
we will say $\hat{h}_n$ is undefined.
They prove the following result for this $\hat{h}_n$, 
by noting that it can be viewed as a stable compression scheme.

\begin{theorem}
\label{thm:original-online-to-batch}
\citep*{bousquet:20b}
Let $\H \subseteq \Y^{\X}$ be 
any nonempty concept class of 
measurable functions, 
let $\alg$ be any conservative online learning algorithm
with $M(\alg,\H) < \infty$,
let $\PXY$ be any distribution on $\X \times \Y$ such 
that $\exists h^* \in \H$ with $R_{\PXY}(h^*)=0$, 
let $n \in \nats$ with $n > 2M(\alg,\H)$, 
and let $S = \{(X_i,Y_i)\}_{i=1}^{n} \sim \PXY^n$.  For $\hat{h}_n$ as defined above, 
for any $\delta \in (0,1)$, with probability at least 
$1-\delta$, 
\begin{equation*}
R_{\PXY}(\hat{h}_n) \leq \frac{2}{n-2M(\alg,\H)} \left( M(\alg,\H) \ln(4) + \ln\!\left(\frac{1}{\delta}\right) \right).    
\end{equation*}
\end{theorem}

The above result matches (up to constants) 
an online-to-batch conversion technique 
of \citet*{littlestone:89}, which was considerably 
more involved (requiring the learner to keep track of 
all intermediate hypotheses, and in the end select one 
of these using a held-out portion of the data).  
Also, the form of the bound in Theorem~\ref{thm:original-online-to-batch} is 
better than analogous PAC bounds known for other 
well-known online-to-batch conversion 
techniques, such as the ``longest survivor'' technique
\citep*{kearns:87,gallant:90} 
or the voting technique \citep*{freund:99}.

While this result is very useful for analyzing certain 
algorithms, there are some online learning algorithms 
for which there are provably bounds on the number 
of mistakes, but only as a function of a property of the 
data sequence.  Such scenarios require an extension 
of this result to allow data-dependent mistake bounds.
An important instance of this is the 
Perceptron algorithm, where a bound on the number of 
mistakes is known, but is quantified in terms of the 
\emph{margin} of the data set; we discuss this in detail 
below.  To extend the online-to-batch conversion 
result to cover these scenarios as well, we may apply 
our Theorem~\ref{thm:realizable}, which allows for 
data-dependent compression sizes.  

Specifically, for any $n \in \nats$ and 
$S = \{(X_i,Y_i)\}_{i=1}^{n} \in (\X \times \Y)^n$, 
continuing the notation from above, 
define $M(\alg,S) = \sum_{j=1}^{\infty} \ind[ \alg(\{(X_{t_i},Y_{t_i})\}_{i=1}^{j},X_{t_{j+1}}) \neq Y_{t_{j+1}} ]$.
In other words, $M(\alg,S)$ is the number of mistakes 
$\alg$ would make if we cycle it through the data set $S$ 
indefinitely.  In particular, 
note that $\hat{h}_n$ is well-defined 
as long as $M(\alg,S) < \infty$.
We have the following result.

\begin{theorem}
\label{thm:online-to-batch}
Let $\alg$ be any conservative online learning algorithm, 
let $\PXY$ be any distribution on $\X \times \Y$, 
let $n \in \nats$ and $\delta \in (0,1)$, 
and let $S = \{(X_i,Y_i)\}_{i=1}^{n} \sim \PXY^n$.  
With probability at least $1-\delta$, 
if $M(\alg,S) < n/2$, then for $\hat{h}_n$ as 
defined above, 
\begin{equation*}
R_{\PXY}(\hat{h}_n) \leq \frac{2}{n-2M(\alg,S)} \left( M(\alg,S) \ln(4) + \ln\!\left(\frac{(M(\alg,S)+1)(M(\alg,S)+2)}{\delta}\right) \right).    
\end{equation*}
\end{theorem}
\begin{proof}
For completeness, we briefly outline here an argument of \citet*{bousquet:20b} 
showing that $\alg$ may be expressed as a stable compression
scheme.
Fix any $n \in \nats$ and any 
data set $S = \{(X_i,Y_i)\}_{i=1}^n \in (\X \times \Y)^n$.
From the definition of $M(\alg,S)$, 
let $j_1,\ldots,j_{M(\alg,S)}$ be the subsequence of $\nats$ 
with 
$\ind[ \alg(\{(X_{t_i},Y_{t_i})\}_{i=1}^{j_s},X_{t_{j_s+1}}) \neq Y_{t_{j_s+1}} ]$, 
and define $j_{0} = 0$ and $j_{M(\alg,S)+1}=\infty$.
Note that, since $\alg$ is conservative, 
for any $s \in \{0,\ldots,M(\alg,S)\}$, 
any $j \in \nats$ with $j_s+1 \leq j \leq j_{s+1}$ has 
$\alg(\{(X_{t_i},Y_{t_i})\}_{i=1}^{j-1},\cdot) 
= \alg(\{(X_{t_{j_{s'}}},Y_{t_{j_{s'}}})\}_{s'=1}^{s},\cdot)$, 
so that removing any points from $S$ that are not among 
$\{(X_{t_j},Y_{t_j})\}_{j=1}^{M(\alg,S)}$ does not change 
the final predictor $\hat{h}_n$.
Thus, $\hat{h}_n$ can be expressed as 
the output of a stable compression scheme: 
namely, $\kappa(S) = \{(X_{t_j},Y_{t_j})\}_{j=1}^{M(\alg,S)}$, 
and for any $S' \in (\X \times \Y)^*$, 
the function $h_{S'}(\cdot)$ produced by $\rho(S')$
is $\alg(S',\cdot)$.

Therefore, for $S \sim \PXY^n$, since $M(\alg,S) < n/2$ 
also implies $\hat{R}_{S}(\hat{h}_n) = 0$, 
and $|\kappa(S)|=M(\alg,S)$, the theorem follows 
immediately from Theorem~\ref{thm:realizable}.
\end{proof}

The bound can also be relaxed to a simpler form, 
with slightly worse numerical constants.

\begin{corollary}
\label{cor:online-to-batch-simple}
Let $\alg$ be any conservative online learning algorithm, 
let $\PXY$ be any distribution on $\X \times \Y$, 
let $n \in \nats$ and $\delta \in (0,1)$, 
and let $S = \{(X_i,Y_i)\}_{i=1}^{n} \sim \PXY^n$.  
With probability at least $1-\delta$, 
if $M(\alg,S) < \infty$ (so that $\hat{h}_n$ is well-defined), 
then for $\hat{h}_n$ as defined above, 
\begin{equation*}
R_{\PXY}(\hat{h}_n) \leq \frac{4}{n} \left( 6 M(\alg,S) + \ln\!\left(\frac{e}{\delta}\right) \right).
\end{equation*}
\end{corollary}

\subsection{Proof of the Optimal PAC Margin Bound for Perceptron}
\label{sec:perceptron-proof}

As was the case for SVM, 
an in-expectation form of the bound was established 
relatively early \citep*{vapnik:74,freund:99},
stating that $\E\!\left[ R_{\PXY}(\perc(S_{n})) \right] = O\!\left( \E\!\left[ \frac{r(S_{n+1})^2}{\gamma(S_{n+1})^2} \frac{1}{n} \right] \right)$, 
where $S_n \sim \PXY^n$ and $S_{n+1} \sim \PXY^{n+1}$.
However, extending the result to an optimal 
data-dependent \emph{PAC} margin bound 
has remained open, since the na\"{i}ve approach based on 
sample compression-based generalization bounds from 
\citet*{littlestone:86,floyd:95} include an extra log factor.
Instead, alternative more-involved 
online-to-batch conversion techniques have been 
needed to obtain the optimal form of the PAC margin bound, 
such as a technique by \citet*{littlestone:89} whereby 
we retain all of the intermediate hypotheses produced 
by the algorithm as it passes through the data, 
and also hold out a portion of the data, using it 
to select which of these intermediate hypotheses 
to return by choosing the one making the fewest mistakes 
on the held-out data.

We prove the following result, from which 
Theorem~\ref{thm:perceptron-margin} immediately follows.

\begin{theorem}
\label{thm:perceptron-full}
For any distribution $\PXY$, any $n \in \nats$, 
and any $\delta \in (0,1)$, for $S \sim \PXY^n$, 
with probability at least $1-\delta$, if $S$ 
is linearly separable, then letting $r = r(S)$ 
and $\gamma = \gamma(S)$, the 
classifier $\hat{h}_p = \perc(S)$ satisfies
\begin{equation*}
R_{\PXY}(\hat{h}_p) \leq \frac{4}{n} \left( 6\frac{r^2+1}{\gamma^2} + \ln\!\left( \frac{e}{\delta} \right) \right)
\end{equation*}
and if $\frac{r^2+1}{\gamma^2} < n/2$ then 
\begin{equation*}
R_{\PXY}(\hat{h}_p) \leq \frac{2}{n - 2(r^2+1)/\gamma^2} \left( \frac{r^2+1}{\gamma^2}\ln(4) + 2\ln\!\left( \frac{r^2+1}{\gamma^2}+2\right) + \ln\!\left( \frac{1}{\delta} \right) \right).
\end{equation*}
\end{theorem}
\begin{proof}
For data lying in a ball of radius $r$
and separable with margin $\gamma$, 
a result of \citet*{MR0175722} implies that the 
conservative online learning algorithm $\perc$ 
makes at most $\frac{r^2+1}{\gamma^2}$ mistakes 
(where the ``$+1$'' is due to the increased radius 
when adding an additional constant-$1$ feature 
to reduce the non-homogeneous case to the 
homogeneous case).
The theorem now immediately follows from 
Theorem~\ref{thm:online-to-batch} 
and Corollary~\ref{cor:online-to-batch-simple}.
\end{proof}

\subsection{An Improved Bound on the Probability in the Region of Disagreement}
\label{sec:PDIS}

Consider now the definitions from Section~\ref{sec:PDIS-results}.

To relate $\Px(\DIS(V_n))$ to $\hat{t}_n$, 
\citet*{hanneke:16b} proved that 
\begin{equation*}
\E\!\left[ \Px(\DIS(V_n)) \right] 
\leq \frac{\E[\hat{t}_{n+1}]}{n+1},
\end{equation*}
based on a leave-one-out argument.
While this bound on the expectation appears 
fairly tight, in contrast 
the analogous known bounds on $\Px(\DIS(V_n))$ 
holding with high probability $1-\delta$ 
each seem to involve some slack.
Specifically, \citet*{hanneke:15a} proved that, 
for any $\delta \in (0,1)$, with probability at least 
$1-\delta$, 
\begin{equation}
\label{eqn:PDIS-original-whey}
\Px( \DIS(V_n) ) \leq \frac{1}{n} \left( 10 
\hat{t}_{n} \ln\!\left( \frac{e n}{\hat{t}_n} \right) + 4 \ln\!\left(\frac{2}{\delta}\right) \right).
\end{equation}
This bound was refined by \citet*{hanneke:16b} 
to remove the factor $\ln\!\left(\frac{ e n }{\hat{t}_n} \right)$, 
but at the expense of larger numerical constants and 
a more-involved dependence on version space compression sizes.
Specifically, \citet*{hanneke:16b} proved that, 
with probability at least $1-\delta$, 
\begin{equation}
\label{eqn:PDIS-original-h}
\Px( \DIS(V_n) ) \leq \frac{16}{n} \left( 2 
\max_{i \leq n} \hat{t}_{i} + \ln\!\left(\frac{3}{\delta}\right) \right).
\end{equation}
The above bound features importantly in obtaining sharp 
distribution-dependent bounds on the label complexity 
of the CAL active learning algorithm 
\citep*{hanneke:16b}.
It is also an important component of 
the analysis of the  
risk of general \emph{empirical risk minimization} 
learning algorithms in traditional (passive) 
supervised learning, established by \citet*{hanneke:16b}.

The original proof of \eqref{eqn:PDIS-original-h} 
by \citet*{hanneke:16b} used the fact that 
the indicator function for $\DIS(V_n)$ can be 
expressed as a sample compression scheme 
(with the compression set being the 
subset of $S_n$ of size $\hat{t}_n$ 
from the definition of $\hat{t}_n$), 
and moreover that $\DIS(V_n)$ is monotonic in $n$.
However, \citet*{bousquet:20b} make the observation 
that this compression scheme is in fact \emph{stable}.
They use this fact to refine numerical constants 
in a particular \emph{distribution-free} bound on 
$\Px(\DIS(V_n))$ from \citet*{hanneke:16b} 
based on a combinatorial complexity measure 
called the \emph{star number} from 
\citet*{hanneke:15b}.  However, they did not 
explore the implications of this observation for 
refining the data-dependent bounds on $\Px(\DIS(V_n))$ 
based on the version space compression set size $\hat{t}_n$.
Here we show that our 
Theorem~\ref{thm:realizable} and 
Corollary~\ref{cor:realizable-simple} 
apply directly to this scenario, 
and offer an immediate improvement to the bounds 
\eqref{eqn:PDIS-original-whey} and \eqref{eqn:PDIS-original-h}
in two respects: namely, we can replace 
$\max_{i \leq n} \hat{t}_i$ with simply $\hat{t}_n$, 
and we can sharpen the numerical 
constant factors in the bound, yielding the 
result claimed in Theorem~\ref{thm:PDIS-coarse}.
Specifically, 
we have the following slightly more-detailed 
result, which immediately implies Theorem~\ref{thm:PDIS-coarse}.

\begin{theorem}
\label{thm:PDIS}
For any $n \in \nats$ and $\delta \in (0,1)$, 
with probability at least $1-\delta$, 
\begin{equation*}
\Px(\DIS(V_n)) \leq \frac{4}{n} \left( 6 \hat{t}_n + \ln\!\left(\frac{e}{\delta}\right) \right)
\end{equation*}
and if $\hat{t}_n < n/2$, 
\begin{equation*}
\Px(\DIS(V_n)) \leq \frac{2}{n-2\hat{t}_n}\left( \hat{t}_n \ln(4) + \ln\!\left(\frac{(\hat{t}_n+1)(\hat{t}_n+2)}{\delta}\right)\right).
\end{equation*}
\end{theorem}
\begin{proof}
Define a function $\hat{h}$ on $\X \times \Y$ 
as $\hat{h}(x,y) = 2\ind_{\DIS(V_n)}(x)-1$, 
and define a distribution $\tilde{\PXY}$ 
on $\X \times \Y \times \Y$ 
such that, 
for $(X,Y,Z) \sim \tilde{\PXY}$,
it holds that $(X,Y) \sim \PXY$, 
and that $Z=-1$ with probability one.
In particular, note that 
$R_{\tilde{\PXY}}(\hat{h}) = \Px(\DIS(V_n))$.
Also, for each $n$, define $\tilde{S}_n = \{ (X,Y,-1) : (X,Y) \in S_n \}$ (retaining the original order 
from $S_n$).
Note that, by the definition of $V_n$, we have 
$\hat{R}_{\tilde{S}_n}(\hat{h}) = 0$.
Furthermore, by the definition of $\hat{t}_n$, 
there exists a subset $S' \subseteq \tilde{S}_n$ 
of size at most $\hat{t}_n$ such that 
$\hat{h}(x,y) = 2 \ind_{\DIS(\H[\{(x,y) : (x,y,-1)\in S'\}])}(x)-1$, 
so that $\hat{h}$ may be viewed as a compression 
scheme $(\kappa,\rho)$ with $|\kappa(\tilde{S}_n)| =  \hat{t}_n$: that is, for any $\tilde{S}$, 
$\kappa(\tilde{S})$ selects any subset $S'$ 
of $\tilde{S}$ of minimum size such that 
$\H[\{ (x,y) : (x,y,-1) \in S' \}] = \H[ \{ (x,y) : (x,y,-1) \in \tilde{S} \} ]$, 
and $\rho(S')(x,y) = 2 \ind_{\DIS(\H[\{(x,y) : (x,y,-1) \in S' \}])}(x)-1$.
Furthermore, clearly any $S'' \subseteq \tilde{S}$ 
with $S' \subseteq S''$ has 
$\H[\{ (x,y) : (x,y,-1) \in S'' \}] = \H[ \{ (x,y) : (x,y,-1) \in \tilde{S} \}]$, 
so that $\rho(\kappa(S'')) = \rho(\kappa(\tilde{S}))$: 
that is, $(\kappa,\rho)$ is a \emph{stable} 
compression scheme.
Therefore, the theorem follows directly from 
Theorem~\ref{thm:realizable} 
and Corollary~\ref{cor:realizable-simple}.
\end{proof}

\begin{remark}
\label{rem:intervals-PDIS}
To illustrate that this can sometimes be a 
significant improvement over the previous 
results \eqref{eqn:PDIS-original-whey} 
and \eqref{eqn:PDIS-original-h}
of \citet*{hanneke:15b} and \citet*{hanneke:16b}, 
respectively, 
consider a scenario where 
$\X = [0,1]$, $\PXY$ has 
$\Px$ uniform on $\X$ 
and for $(X,Y) \sim \PXY$ 
we have $Y = 2 \ind_{[a,b]}(X)-1$, 
for some $a,b \in (0,1)$ with $a < b$.
For $n > \frac{1}{b-a}\ln\!\left(\frac{1}{\delta}\right)$, 
with probability greater than $1-\delta$ we have 
$\hat{t}_n \leq 4$. 
However, for $n < \frac{1}{b-a}$ there is a 
nonzero constant probability that all 
samples have negative labels, in which case 
$\hat{t}_n = n$.  
Thus, for all large values of $n$, 
the bound in Theorem~\ref{thm:PDIS} 
is $O\!\left( \frac{1}{n} \log\!\left(\frac{1}{\delta}\right) \right)$,
whereas the bound 
\eqref{eqn:PDIS-original-whey} is 
$\Omega\!\left( \frac{1}{n} \log\!\left(\frac{n}{\delta}\right) \right)$
and the bound 
\eqref{eqn:PDIS-original-h}
is 
$\Omega\!\left( \frac{1}{n} \left( \frac{1}{b-a} + \log\!\left(\frac{1}{\delta}\right) \right) \right)$.
Thus, the improvement in Theorem~\ref{thm:PDIS} 
can be quite significant when $b-a$ is relatively 
small, and $n$ is large relative to $\frac{1}{b-a}$.
\end{remark}

\subsection{Proof of the Improved Data-dependent Bound for All ERM Algorithms}
\label{sec:ERM-proof}

Here we present the details related to Theorem~\ref{thm:ERM-coarse}: 
data-dependent risk bounds holding for 
all ERM learning algorithms.

The result in Theorem~\ref{thm:ERM-coarse} 
improves over a previous result of 
\citet*{hanneke:16b}, which states that, 
with probability at least $1-\delta$, 
every $h \in V_n$ satisfies 
\begin{equation}
\label{eqn:ERM-h-bound}
R_{\PXY}(h) = O\!\left( \frac{1}{n} \left( \vc \log\!\left(\frac{\max_{i \leq n} \hat{t}_i}{\vc} \right) + \log\!\left(\frac{1}{\delta}\right) \right) \right).
\end{equation}
Comparing the two bounds reveals that 
the improvement in Theorem~\ref{thm:ERM-coarse} 
is in replacing $\max_{i \leq n} \hat{t}_i$ 
with $\hat{t}_{\lfloor n/2 \rfloor}$.
Recalling Remark~\ref{rem:intervals-PDIS} above, 
this change can sometimes be significant.
In particular, in the example discussed in that 
remark, for large $n$ the bound in 
Theorem~\ref{thm:ERM-coarse} would be 
$O\!\left( \frac{1}{n} \log\!\left(\frac{1}{\delta}\right) \right)$, 
whereas the previously known 
bound from \eqref{eqn:ERM-h-bound} 
would be $\Omega\!\left( \frac{1}{n} \log\!\left(\frac{1}{(b-a)\delta}\right)\right)$.
Thus, in this example, Theorem~\ref{thm:ERM-coarse} 
reflects a significant improvement when 
$b-a$ is small, 
and $n$ is large relative to $\frac{1}{b-a}$.

The proof of Theorem~\ref{thm:ERM-coarse} 
follows identical arguments to those used 
by \citet*{hanneke:16b} to prove 
\eqref{eqn:ERM-h-bound}, aside from 
substituting the bound on 
$\Px(\DIS(V_{\lfloor n/2 \rfloor}))$ 
from Theorem~\ref{thm:PDIS} in place 
of the bound \eqref{eqn:PDIS-original-h} 
also established by \citet*{hanneke:16b}.
As such we omit the details, and merely 
sketch the main ideas underlying the 
argument.

The proof follows a ``conditioning'' 
argument common to this literature on refining 
log factors in risk bounds, originating 
in a proof from \citet*{hanneke:thesis} 
of a related distribution-dependent 
bound for all ERM learners
(which itself is slightly looser than \eqref{eqn:ERM-h-bound}; 
see \citealp*{hanneke:15b} and \citealp*{hanneke:16b}
for relations between the relevant complexity measures).
The high-level idea is to note that 
for $S_n = \{(X_i,Y_i)\}_{i=1}^n$, 
the set 
$D_n := \{ (X_i,Y_i) : i > n/2, X_i \in \DIS(V_{\lfloor n/2 \rfloor}) \}$
contains roughly $(n/2) \Px(\DIS(V_{\lfloor n/2 \rfloor}))$ elements (with high probability), 
which are conditionally i.i.d.\ with distribution 
$Q := \Px(\cdot | X \in \DIS(V_{\lfloor n/2 \rfloor}))$
given $S_{\lfloor n/2 \rfloor}$ and $|D_n|$.
Furthermore, any $h \in V_n$ has 
$\hat{R}_{D_n}(h) = 0$.
Thus, applying classic generalization bounds for 
ERM from \citet*{vapnik:74,blumer:89} implies that
(with high probability) every $h \in V_n$
has 
\begin{align*}
R_Q(h) &
= O\!\left( \frac{1}{|D_n|} \left( \vc \log\!\left(\frac{|D_n|}{\vc}\right) + \log\!\left(\frac{1}{\delta}\right)\right)\right)
\\ & = O\!\left( \frac{1}{n \Px(\DIS(V_{\lfloor  n/2 \rfloor}))} \left( \vc \log\!\left(\frac{n\Px(\DIS(V_{\lfloor  n/2 \rfloor}))}{\vc}\right) + \log\!\left(\frac{1}{\delta}\right)\right)\right).
\end{align*}
Since every $h \in V_n$ agrees with the 
\emph{best} classifier $h^* \in \H$ 
on $\X \setminus \DIS(V_{\lfloor n/2 \rfloor})$, 
we have $R_{\PXY}(h) = R_{Q}(h) \Px(\DIS(V_{\lfloor n/2 \rfloor}))$, so that 
\begin{equation*}
R_{\PXY}(h) 
= O\!\left( \frac{1}{n} \left( \vc \log\!\left(\frac{n\Px(\DIS(V_{\lfloor  n/2 \rfloor}))}{\vc}\right) + \log\!\left(\frac{1}{\delta}\right)\right)\right).
\end{equation*}
Finally, plugging in the bound on 
$\Px(\DIS(V_{\lfloor n/2 \rfloor}))$ 
from Theorem~\ref{thm:PDIS}, together with a union 
bound over the above high-probability events, 
and simplifying the expression, we get 
that (with high probability) each $h \in V_n$ 
satisfies 
\begin{equation*}
R_{\PXY}(h) 
= O\!\left( \frac{1}{n} \left( \vc \log\!\left(\frac{\hat{t}_{\lfloor n/2 \rfloor}}{\vc}\right) + \log\!\left(\frac{1}{\delta}\right)\right)\right),
\end{equation*}
which is the bound claimed by Theorem~\ref{thm:ERM-coarse}.
Readers interested in the details (and handling 
corner cases, numerical constants, and such things) 
are refered to the detailed proof of 
\eqref{eqn:ERM-h-bound} by \citet*{hanneke:16b}.

\subsection{Proof of the Improved Bound for Compressed 1-Nearest Neighbor}
\label{sec:optinet-proof}

We present here the details regarding 
Theorem~\ref{thm:NN-bernstein}: 
the a generalization bound for 
the compression-based nearest neighbor predictor.
We have the following result, 
from which Theorem~\ref{thm:NN-bernstein} 
immediately follows.

\begin{theorem}
\label{thm:NN-bernstein-full}
Fix any $\gamma > 0$.
For any distribution $\PXY$, any $n \in \nats$, 
and any $\delta \in (0,1)$, 
for $S \sim \PXY^{n}$, 
with probability at least $1-\delta$, 
the classifier $\hat{h}_{\gamma} = \alg_{\gamma}(S)$ 
satisfies 
\begin{align*}
\left| R_{\PXY}(\hat{h}_{\gamma}) - \hat{R}_{S}(\hat{h}_{\gamma}) \right| 
\leq \sqrt{ \hat{R}_{S}(\hat{h}_{\gamma}) \frac{72}{n} \left( 4|N_{\gamma}| + \ln\!\left(\frac{4e}{\delta}\right) \right) }
+ \frac{32}{n} \left( 4|N_{\gamma}| + \ln\!\left(\frac{4e}{\delta}\right) \right).
\end{align*}
\end{theorem}
\begin{proof}
We prove this by arguing that $\hat{h}_{\gamma}$ 
is the output of $\rho(\kappa(S))$ 
for some stable compression scheme $(\kappa,\rho)$ 
with $|\kappa(S)|=|N_{\gamma}|$ contained in a 
particular family of such compression schemes.
Specifically, for each $k \in \nats$ 
and $\mathbf{b} = (b_1,\ldots,b_k) \in \Y^k$, 
define a compression scheme $(\kappa,\rho_{\mathbf{b}})$
such that $\kappa(S) = N_{\gamma}$ (the $\gamma$-net 
corresponding to $S$), and 
$\rho_{\mathbf{b}}(\kappa(S))$ is a function such that, 
for any $x$, if the nearest neighbor of $x$ 
among $N_{\gamma}$ is the $i^{{\rm th}}$ element 
of $N_{\gamma}$, then if $i \leq k$ 
then $\rho_{\mathbf{b}}(\kappa(S)) = b_i$, 
and otherwise if $i > k$ then $\rho_{\mathbf{b}}(\kappa(S)) = -1$.
Since $\kappa(S)$ stable, in the sense that 
any $S' \subset S$ with $\kappa(S) \subseteq S'$ 
has $\kappa(S') = \kappa(S)$, 
it follows that $(\kappa,\rho_{\mathbf{b}})$ 
is a stable compression scheme.
Thus, for each $k \in \nats$ 
and $\mathbf{b} \in \Y^k$, 
Corollary~\ref{cor:bernstein-simple} implies 
that, with probability at least $1 - \frac{\delta}{2^{2k}}$, it holds that 
\begin{align*}
& \left| R_{\PXY}(\rho_{\mathbf{b}}(\kappa(S))) - \hat{R}_{S}(\rho_{\mathbf{b}}(\kappa(S))) \right| 
\\ & \leq \sqrt{ \hat{R}_{S}(\rho_{\mathbf{b}}(\kappa(S))) \frac{72}{n} \left( 2|N_{\gamma}|+2k + \ln\!\left(\frac{4e}{\delta}\right) \right) }
+ \frac{32}{n} \left( 2|N_{\gamma}|+2k + \ln\!\left(\frac{4e}{\delta}\right) \right).
\end{align*}
By the union bound, this holds simultaneously 
for all $k \in \nats$ and $\mathbf{b} \in \Y^k$, 
with probability at least $1-\delta$.
In particular, note that we are guarnateed 
that $\hat{h}_{\gamma} = \rho_{\mathbf{b}}(\kappa(S))$
for some $\mathbf{b} \in \Y^{|N_{\gamma}|}$.
Thus, with probability at least $1-\delta$, 
\begin{align*}
& \left| R_{\PXY}(\hat{h}_{\gamma}) - \hat{R}_{S}(\hat{h}_{\gamma}) \right| 
= 
\left| R_{\PXY}(\rho_{\mathbf{b}}(\kappa(S))) - \hat{R}_{S}(\rho_{\mathbf{b}}(\kappa(S))) \right|
\\ & 
\leq \sqrt{ \hat{R}_{S}(\hat{h}_{\gamma}) \frac{72}{n} \left( 4|N_{\gamma}| + \ln\!\left(\frac{4e}{\delta}\right) \right) }
+ \frac{32}{n} \left( 4|N_{\gamma}| + \ln\!\left(\frac{4e}{\delta}\right) \right).
\end{align*}
\end{proof}

\bibliography{learning}

\end{document}

%% file: macros.tex
\newlength{\dhatheight}

\newcommand{\Px}{P_{X}}
\newcommand{\PXY}{P}

\newcommand{\J}{\mathcal{J}}

\newcommand{\vc}{{\rm d}}

\renewcommand{\dim}{d} % to be introduced as short-hand for VC dimension
\renewcommand{\epsilon}{\varepsilon}
\newcommand{\eps}{\varepsilon}

\newcommand{\X}{\mathcal X} % instance space
\newcommand{\Y}{\mathcal Y} % label space
\newcommand{\alg}{\mathcal{A}} % learning algorithm
\renewcommand{\H}{\mathcal{H}} % hypothesis space
\DeclareSymbolFont{bbold}{U}{bbold}{m}{n}
\DeclareSymbolFontAlphabet{\mathbbold}{bbold}
\newcommand{\ind}{\mathbbold{1}}
\newcommand{\sign}{{\rm sign}} % sign function
\renewcommand{\P}{{\rm Pr}} % also probability
\newcommand{\I}{\mathcal{I}}
\newcommand{\nats}{\mathbb{N}} % natural numbers
\newcommand{\reals}{\mathbb{R}} % real numbers

\newcommand{\E}{\mathbb{E}}

\newcommand{\DIS}{{\rm DIS}}

\newcommand{\ignore}[1]{}
\newcommand{\oldstuff}[1]{}

\newcommand{\new}[1]{}
\newcommand{\sh}[1]{}
\newcommand{\hide}[1]{}

\newcommand{\paren}[1]{\left( #1 \right)}

\newcommand{\tlprn}[1]{\left\{ #1 \right\}}
\renewcommand{\set}[1]{\tlprn{#1}}
\newcommand{\R}{\mathbb{R}}

\newcommand{\beq}{\begin{eqnarray*}}
\newcommand{\eeq}{\end{eqnarray*}}
\newcommand{\beqn}{\begin{eqnarray}}
\newcommand{\eeqn}{\end{eqnarray}}

\newsavebox{\savepar}

\makeatletter
\newcommand{\vast}{\bBigg@{3}}
\newcommand{\Vast}{\bBigg@{4}}
\makeatother

\newcommand{\x}{\vec{x}}

\newcommand{\w}{\vec{w}}

\renewcommand{\dim}{d}
\renewcommand{\phi}{\varphi}

\usepackage{bbm}
\newcommand{\perc}{\alg_p}

%% file: abstract.tex
We analyze a family of supervised learning algorithms based on sample compression schemes that are {\em stable}, in the sense that removing points from the training set which were not selected for the compression set does not 
alter
the resulting classifier.  We use this technique to derive a variety of novel or improved data-dependent generalization bounds for several learning algorithms.  In particular, we prove a new margin bound for SVM, removing a log factor. The new bound is provably optimal.  This resolves a 
long-standing open question about the PAC margin bounds achievable by SVM.